\documentclass{article}
\usepackage[final]{neurips_2022}
\usepackage[utf8]{inputenc} % allow utf-8 input
\usepackage[T1]{fontenc}    % use 8-bit T1 fonts
\usepackage{hyperref}       % hyperlinks
\usepackage{url}            % simple URL typesetting
\usepackage{booktabs}       % professional-quality tables
\usepackage{amsfonts}       % blackboard math symbols
\usepackage{nicefrac}       % compact symbols for 1/2, etc.
\usepackage{microtype}      % microtypography
\usepackage{xcolor}         % colors
\usepackage{multirow}
\PassOptionsToPackage{options}{natbib}
\setcitestyle{numbers,open={[},close={]}}
\usepackage{microtype}
\usepackage{graphicx}
\usepackage{subfigure}
\usepackage{booktabs} 
\usepackage{amsmath}
\usepackage{mathtools}
\usepackage{amssymb,amsfonts}
\usepackage{hyperref}
\usepackage{authblk}
\usepackage{algorithm}
\usepackage{algorithmic}
\makeatletter

\newcommand{\Rmnum}[1]{\expandafter\@slowrom\mathtt{anc}ap\romannumeral #1@}
\makeatother
\newtheorem{theorem}{Theorem}
\newtheorem{lemma}{Lemma}
\newtheorem{claim}{Claim}
\newtheorem{definition}{Definition}
\newtheorem{corollary}{Corollary}
\newtheorem{proposition}{Proposition}
\newtheorem{assumption}{Assumption}
\newtheorem{remark}{Remark}
\newenvironment{proof}{{\noindent\it Proof}.}{\hfill $\square$\par}

\PassOptionsToPackage{numbers, compress}{natbib}

\title{
  Coresets for Wasserstein Distributionally Robust Optimization Problems}

\author[1]{\textbf{Ruomin Huang}}

\author[2,3]{\textbf{Jiawei Huang}}

\author[2]{\textbf{Wenjie Liu}}

\author[2]{\textbf{Hu Ding\thanks{Corresponding author.}}\ \ }

\affil[1]{School of Data Science} \affil[2]{School of Computer Science and Technology}
\affil[ ]{University of Science and Technology of China}
\affil[3]{Department of Computer Science, City University of Hong Kong}
\affil[ ]{\texttt{\{hrm, hjw0330, lwj1217\}}\texttt{@mail.ustc.edu.cn}, \texttt{\href{mailto:huding@ustc.edu.cn}{huding@ustc.edu.cn}}}

\begin{document}

\maketitle

\begin{abstract}
Wasserstein distributionally robust optimization (\textsf{WDRO}) is a popular model to enhance the robustness of machine learning with ambiguous data. However, the complexity of \textsf{WDRO} can be prohibitive in practice since solving its ``minimax'' formulation requires a great amount of computation. Recently, several fast \textsf{WDRO} training algorithms for some specific machine learning tasks (e.g., logistic regression) have been developed.
However, the research on designing efficient algorithms for general large-scale \textsf{WDRO}s is still quite limited, to the best of our knowledge. \textit{Coreset} is an important  tool for compressing large dataset, and thus it has been widely applied to  reduce the computational complexities for many optimization problems. In this paper, we introduce a unified framework to construct the $\epsilon$-coreset for the general \textsf{WDRO} problems. Though it is challenging to obtain a conventional coreset for \textsf{WDRO}  due to the   uncertainty issue of ambiguous data, we show that we can compute a ``dual coreset'' by using the strong duality property of \textsf{WDRO}. Also, the error introduced by the dual coreset can be theoretically guaranteed for the original \textsf{WDRO} objective. To construct the dual coreset, we propose a novel  grid sampling approach that is particularly suitable for the dual formulation of \textsf{WDRO}. Finally, we implement our coreset approach and illustrate its effectiveness for several \textsf{WDRO} problems in the experiments. 
\end{abstract}

\section{Introduction}
In the past decades, a number of optimization techniques have been proposed for solving machine learning problems~\citep{sra2012optimization}. However, real-world optimization problems often suffer from the issue of data ambiguity that can be generated by natural data noise,   potential adversarial attackers~\citep{DBLP:journals/pr/BiggioR18}, or the constant changes of the underlying distribution (e.g., continual learning~\citep{ring1998child}).  As a consequence, our obtained dataset usually cannot be fully trusted. Instead it is actually a perturbation of the true distribution. The recent studies have shown that even small perturbation can seriously destroy the final optimization result and could also yield unexpected error for the applications like classification and pattern recognition~\cite{goodfellow2018making,DBLP:journals/corr/SzegedyZSBEGF13}. 
%Therefore, robust optimization has attracted a great amount of attention in recent years. 

The ``\textbf{distributionally robust optimization} (\textsf{DRO})'' is an elegant model for solving the issue of ambiguous data. The idea follows from the intuition of   game  theory~\citep{rahimian2019distributionally}. Roughly speaking, the \textsf{DRO} aims to find a solution that is robust against the worst-case perturbation within a range of possible distributions. 
Given an empirical distribution $\mathbb{ P}_n=\frac{1}{n}\sum\limits_{i=1}^n\delta_{\xi_i}$ where $\delta_{\xi_i}$ is the Dirac point mass at the $i$-th data sample $\xi_i$, the \textbf{worst-case empirical risk} at the hypothesis $\theta$ is defined as $R^{\mathbb{P}_n}(\theta) = \sup\limits_{\mathbb{Q}\in \mathcal{U}(\mathbb{P}_n)} \mathbb{E}^\mathbb{Q}[\ell(\theta,\xi)]$. 
%\begin{eqnarray}
%\label{eq:worst-case}
%R^{\mathbb{P}_n}(\theta) = \sup\limits_{\mathbb{Q}\in \mathcal{U}(\mathbb{P}_n)} \mathbb{E}^\mathbb{Q}[\ell(\theta,\xi)].
%\end{eqnarray}
Here $\mathcal{U}(\mathbb{P}_n)$ is the ambiguity set consisting of all possible  distributions of interest, and $\ell(\cdot,\cdot)$ is the non-negative loss function. The DRO model has shown its promising advantage for enhancing the robustness for many practical machine learning problems, such as logistic regression~\citep{shafieezadeh2015distributionally}, support vector machine~\citep{lee2015distributionally},
convex regression~\citep{blanchet2019multivariate}, neural networks~\citep{sagawa2019distributionally, DBLP:conf/iclr/SinhaND18}, etc.

In this paper,  we consider one of the most representative \textsf{DRO} models that is defined by using optimal transportation ~\citep{villani2009optimal}.  \textbf{Wasserstein distance} is a popular measure for representing the difference between two distributions; it indicates the minimum cost for transporting one distribution to the other. 
%\begin{definition}[Wasserstein distance]
For $p\geq 1$, the $p$-th order Wasserstein distance between two probability distributions $\mathbb{P}$ and $\mathbb{P'}$ supported on $\Xi$ is 
\begin{eqnarray}
W_p(\mathbb{P},\mathbb{P'})=\left(\inf _{\pi \in \Pi\left(\mathbb{P}, \mathbb{P}^{\prime}\right)} \int_{\Xi\times\Xi} \mathtt{d}^p(\xi,\xi^{\prime})\pi\left(\mathrm{d} \xi, \mathrm{d} \xi^{\prime}\right)\right)^{\frac{1}{p}},\label{for-wasdist}
\end{eqnarray}
where $\mathtt{d}(\cdot,\cdot)$ is a metric on $\Xi$, and $\Pi(\mathbb{P},\mathbb{P}^{\prime})$ is the set of all joint probability distributions on $\Xi\times\Xi$ with the  marginals $\mathbb{P}$ and $\mathbb{P}'$.
%\end{definition}
%{\color{red} + the formula for p-th order wasserstein distance.} 
By using the above Wasserstein distance~(\ref{for-wasdist}), we can define the ambiguity set   $\mathcal{U}(\mathbb{P}_n)$ to be the $p$-th order \textbf{Wasserstein  ball} $\mathbb{B}_{\sigma,p}(\mathbb{P}_n)$, which covers all the distributions that have the $p$-th order Wasserstein distance at most $\sigma>0$ to the given empirical distribution $\mathbb{P}_n$.  The use of Wasserstein ball is a discrepancy-based approach for choosing the ambiguity set \cite[section 5]{rahimian2019distributionally}. 
Also let 
%$R^{\mathbb{P}_n}_{\sigma,p}(\theta)= \sup\limits_{\mathbb{Q}\in \mathbb{B}_{\sigma,p}(\mathbb{P}_n)} \mathbb{E}^\mathbb{Q}[\ell(\theta,\xi)]$ 
\begin{eqnarray}
\label{eq:worst-case}
R^{\mathbb{P}_n}_{\sigma,p}(\theta)= \sup\limits_{\mathbb{Q}\in \mathbb{B}_{\sigma,p}(\mathbb{P}_n)} \mathbb{E}^\mathbb{Q}[\ell(\theta,\xi)]
\end{eqnarray}
denote the corresponding worst-case empirical risk. The \textbf{Wasserstein distributionally robust optimization} (\textsf{WDRO}) problem~\cite{kuhn2019wasserstein} is to find the minimizer 
\begin{eqnarray}
\label{eq:WDRO}
\theta_*=\mathop{\arg\min}\limits_{\theta\in\Theta}R^{\mathbb{P}_n}_{\sigma,p}(\theta)=\mathop{\arg\min}\limits_{\theta\in\Theta}\sup\limits_{\mathbb{Q}\in \mathbb{B}_{\sigma,p}(\mathbb{P}_n)} \mathbb{E}^\mathbb{Q}[\ell(\theta,\xi)],
\end{eqnarray}
where $\Theta$ is the feasible region in the hypothesis space. It is easy to see that the \textsf{WDRO} is a minimax optimization problem.

%\textbf{The virtue of WDRO.} Like other DROs, WDRO enjoys a guarantee for the out-of-sample performance ~\citep{DBLP:journals/mp/EsfahaniK18}.  Compared with other types of ambiguity sets, the Wasserstein type can be theoretically more related to machine learning since WDRO is equivalent to optimizations with popular norm regularization in some special cases ~\citep{xu2009robustness,shafieezadeh2019regularization,blanchet2019robust,DBLP:journals/corr/abs-1712-06050}. Besides the well-known advantage that the Wasserstein ball contains rich enough relevant distribution. The Wasserstein metric also considers the distance between data points, capturing more information than the divergence-based discrepancies in some missions like pattern recognition and image retrieval problems. ~\citep{rubner2000earth,ling2007efficient,gao2016distributionally}.

Compared with other robust optimization models, the \textsf{WDRO} model enjoys several significant benefits from the Wasserstein metric, especially for the applications in machine learning~\citep{xu2009robustness,shafieezadeh2019regularization,blanchet2019robust,DBLP:journals/corr/abs-1712-06050}. The Wasserstein ball  captures much richer information than the divergence-based discrepancies for the problems like pattern recognition and image retrieval~\citep{rubner2000earth,ling2007efficient,gao2016distributionally}. It has also been proved that the \textsf{WDRO} model yields  theoretical quality guarantees for the ``out-of-sample'' robustness~\citep{articleDRO}.

%\textbf{Tractable reformulations and specialized algorithms.} 

However, due to the intractability of the the inner maximization problem (\ref{eq:worst-case}), it is challenging to directly solve the minimax optimization problem (\ref{eq:WDRO}). As shown in the work of~\citet{articleDRO}, the \textsf{WDRO} problem  (\ref{eq:WDRO}) usually has tractable reformulations~\citep{shafieezadeh2015distributionally,mohajerin2018data,postek2016computationally,lee2015distributionally,hanasusanto2018conic,blanchet2019quantifying}. Although these reformulations  are polynomial-time solvable, the off-the-shelf solvers can be costly for large-scale data. Another approach  is to directly solve the minimization problem and the maximization problem alternatively~\citep{pflug2014problem} under a finite-support assumption. \citet{gao2016distributionally} proposed a routine to compute the finite structure of the worst-case distribution in theory. Nevertheless it still takes a high computational complexity if the Wasserstein ball has a large support size. 
%Here a finite-support (or fixed-support) assumption on probability distributions is imperative to tackle the maximization problem (\ref{eq:worst-case}). The work of \citet{gao2016distributionally} reveals the finite structure of the worst-case distribution, which justifies the alternative optimization method to a certain extent. Although this method can handle general loss functions and does not require tractable reformulations, it could be time-consuming if the support size becomes large. 
Several fast \textsf{WDRO} training algorithms for some specific machine learning tasks, e.g., SVM  and logistic regression by \citet{li2019first,li2020fast}, have been developed recently; but it is unclear whether their methods can be generalized to solve other problems.

%Inspired by the corresponding convex programming reformulations, researchers proposed several algorithms specialized for WDRO with certain machine learning model, e.g., support vector machine ~\citep{lee2015distributionally,li2020fast} and logistic regression ~\citep{li2019first}. While the algorithms of  \citet{li2019first} and \citet{li2020fast} can respectively handle the large-scale logistic regression and support vector machine, the general method for WDRO which scales well with problem size still remains to investigate.  

Therefore, it is urgent to develop efficient algorithmic techniques for reducing the computational complexity of the \textsf{WDRO} problems.  \textbf{Coreset} is a popular tool for compressing large datasets, which  was initially introduced by Agarwal et al. in computational geometry~\citep{DBLP:journals/jacm/AgarwalHV04}. Intuitively, the coreset is an approximation of the original input data, but has a much smaller size. Thus any existing algorithm can run on the coreset instead and the computational complexity can be largely reduced. The coresets techniques have been widely applied for many optimization problems such as clustering and  regression (we refer the reader to the recent surveys on coresets~~\citep{DBLP:journals/ki/MunteanuS18,DBLP:journals/widm/Feldman20}). 
Therefore a natural idea is to consider applying the coreset technique to deal with large-scale \textsf{WDRO} problems.  
%From the perspective of traditional coreset literature, 
Below we introduce the formal definition of the coreset for \textsf{WDRO} problems.

\begin{definition}[$\epsilon$-coreset]
\label{def:coreset}
Let $\epsilon$ be any given small number in $(0,1)$. An $\epsilon$-coreset for the \textsf{WDRO} problem (\ref{eq:WDRO}) is a sparse nonnegative mass vector $W=[w_1,\dots,w_n]$, such that the total mass $\sum_{i=1}^nw_i=1$ and the induced distribution $\tilde{\mathbb{P}}_n=\sum\limits_{i=1}^nw_i\delta_{\xi_i}$ satisfies 
\begin{eqnarray}
R^{\mathbb{\tilde P}_n}_{\sigma,p}(\theta)\in (1\pm\epsilon) R^{\mathbb{P}_n}_{\sigma,p}(\theta), \forall \theta\in \Theta,
\end{eqnarray}
where $R_{\sigma,p}^{\mathbb{\tilde P}_n}(\theta)\coloneqq\sup\limits_{\mathbb{Q}\in \mathbb{B}_{\sigma,p}(\tilde{\mathbb{P}}_n)} \mathbb{E}^\mathbb{Q}[\ell(\theta,\xi)]$ is the worst-case empirical risk of the coreset. 
\end{definition}

It is worth to emphasize that the above coreset for \textsf{WDRO} is fundamentally different from the conventional  coresets~\citep{DBLP:journals/widm/Feldman20}. The main challenge for constructing the coreset of \textsf{WDRO} is from the ``\textbf{uncertainty}'' issue, that is, we have to consider all the possible distributions in the Wasserstein ball $\mathbb{B}_{\sigma,p}(\tilde{\mathbb{P}}_n)$; and more importantly, when the parameter vector $\theta$ varies, the distribution that achieves the worst-case empirical risk also changes inside $\mathbb{B}_{\sigma,p}(\tilde{\mathbb{P}}_n)$.  

%However, like the intractability of maximization problem (\ref{eq:worst-case}), it is not easy to verify whether a given vector $W$ is an $\epsilon$-coreset or not, let alone generating one. 

%\textcolor{red}{background for WDRO, why WDRO}

%\textcolor{red}{large scale data requires coreset, WDRO heavily depends on the size of data; large data does not guarantee the dataset can represent the underlying distribution since attack or the underlying distribution varies as time changes}

%\textcolor{red}{background for coreset, existing coreset cannot handle WDRO since infinite dimensional and inner dependence among data}
\subsection{Our Contribution}
In this paper, we propose a novel framework to  construct the $\epsilon$-coresets for general \textsf{WDRO} problems. To the best of our knowledge, this is the first coreset algorithm for Wasserstein distributionally robust optimization problems. Our main contributions are twofold.  

%the main contributions can be concluded as the following.

 \textbf{-From coresets to dual coresets.} As mentioned before, it is challenging to directly construct the coresets for the \textsf{WDRO} problems. Our key observation is inspired by the strong duality property of the \textsf{WDRO} model ~\citep{articleDRO,blanchet2019quantifying,gao2016distributionally}.  We introduce the ``dual coreset'' for the dual formulation of the \textsf{WDRO} problems. We can neatly circumvent  the ``uncertainty'' issue in Definition~\ref{def:coreset}   through  the dual form.  Also, we prove that the dual coreset  can yield a theoretically quality-guaranteed coreset as Definition~\ref{def:coreset}. 

\textbf{-How to compute the dual coresets.} Further, we provide a unified framework to construct the dual coresets efficiently. The sensitive-sampling based coreset framework usually needs to compute the ``pseudo-dimension'' of the objective function and the ``sensitivities'' of the data items, which can be very difficult to obtain~~\cite{DBLP:journals/widm/Feldman20} (the pseudo-dimension measures how complicated the objective function is, and the sensitivity of each data item indicates its importance to the whole input data set). Therefore we consider to apply the spatial partition approach that was initiated by Chen~\citep{chen2009coresets}; roughly speaking, we partition the space into a logarithmic number of regions, and take a uniform sample from each region. This partition approach needs to compute the exact value of the  Moreau-Yosida regularization~\citep{parikh2014proximal},  which is a key part in the dual formulation of \textsf{WDRO} (the formal definition is shown in Proposition~\ref{pro:duality}). However, this value is often  hard to obtain for general $\Xi$ and general $\ell(\cdot,\cdot)$. For instance, suppose $\Xi$ admits a conic representation and the learning model is SVM, then computing the Moreau-Yosida regularization is equivalent to solving a convex conic programming ~\citep[corollary 3.12]{shafieezadeh2019regularization}. For some machine learning problems, it is usually relatively easier to estimate the bounds of the Moreau-Yosida regularization\cite[Theorem 3.30]{shafieezadeh2019regularization}. 
%; if the learning model is deep neural network, the exact value of the Moreau-Yosida regularization is hard to obtain but an upper bound is available \cite[Theorem 3.30]{shafieezadeh2019regularization}. 
Based on this observation, we generalize the spatial partition idea and propose a more practical ``grid sampling'' framework. By using this framework, we only need to estimate the upper and lower bounds of the Moreau-Yosida regularization instead of the exact value. We also prove that a broad range of objective functions can be handled under this framework.

\subsection{Other Related Works}

%We briefly introduce several other related works below. 

%\textbf{Large-scale DRO algorithms.} For the moment-based ambiguity set ({\color{red} slightly explain what is moment-based?}), \citet{DBLP:journals/siamjo/ChengCNPSW18} and \citet{cheramin2020computationally} respectively applied PCA to reduce the high-dimensional data size for large-scale DRO problems. 
%a dimension reduction method for DRO based on principle component analysis (PCA) to deal with large-scale data. 
%For the combined ambiguity set where there are restrictions on both moments and the Wasserstein distance, \citet{cheramin2020computationally} develope an computationally efficient approximations based on PCA as well. 
%For the $\phi$-divergence ambiguity set, stochastic gradient optimization methods are utilized to handle the large-scale data ~\citep{DBLP:conf/nips/LevyCDS20,DBLP:conf/nips/NamkoongD16}.

A number of coreset-based techniques have been studied before for solving robust optimization problems. For example, \citet{DBLP:conf/nips/MirzasoleimanCL20} designed an algorithm to generate coreset to approximate the Jacobian of a neural network so as to train against noisy labels. The outlier-resistant coresets were also studied for computing the robust center-based clustering problems~\citep{DBLP:conf/stoc/FeldmanL11,DBLP:conf/soda/FeldmanS12,DBLP:conf/focs/HuangJLW18,DBLP:conf/esa/DingYW19}. 
For the general continuous and bounded optimization problems~\cite{understandingML}, \citet{DBLP:conf/nips/WangGD21} proposed a dynamic framework to compute the coresets resisting against outliers. Several other techniques also have been proposed for dealing with large-scale \textsf{DRO} problems, such as the PCA based dimension reduction methods~\citep{DBLP:journals/siamjo/ChengCNPSW18,cheramin2020computationally} and the stochastic gradient optimization methods~\citep{DBLP:conf/nips/LevyCDS20,DBLP:conf/nips/NamkoongD16}.

\section{Preliminaries}
\label{sec:pre}

We assume the input-output space $\Xi=\mathbb{X}\times\mathbb{Y}$ with $\mathbb{X}\subseteq\mathbb{R}^m$ and $\mathbb{Y}\subseteq\mathbb{R}$, and let $\mathcal{P}(\Xi)$ denote the set of Borel probability distributions supported on $\Xi$. For $1\leq i\leq n$, each data sample is a random vector $\xi_i=(x_i, y_i)$ drawn from some underlying distribution $\mathbb{P}\in \mathcal{P}(\Xi)$.
%In this paper, we adopt the online optimization setting where a data sample is a random vector $\xi=(x,y)$ drawn from some underlying distribution $\mathbb{P}\in \mathcal{P}(\Xi)$. Here $\Xi=\mathbb{X}\times\mathbb{Y}$ denotes the input space and $\mathcal{P}(\Xi)$ denotes the set of Borel probability distributions supported on $\Xi$. The input space $\Xi$ can be an arbitrary convex set and $\mathbb{X}\subset\mathbb{R}^m,\mathbb{Y}\subset\mathbb{R}$. 
The empirical distribution $\mathbb{P}_n=\frac{1}{n}\sum_{i=1}^n\delta_{\xi_i}$ is induced by the dataset $\{\xi_1,\dots,\xi_n\}$, where $\delta_{\xi_i}$ is the Dirac point mass at  $\xi_i$. We endow $\Xi$ with the feature-label metric $\mathtt{d}(\xi_i,\xi_j)=\|x_i-x_j\|+\frac{\gamma}{2}|y_i-y_j|$, where $\|\cdot\|$ stands for an arbitrary norm of $\mathbb{R}^m$ and the positive parameter ``$\gamma$'' quantifies the transportation cost on the label. This distance function is used for defining the Wasserstein distance (\ref{for-wasdist}). We assume that $(\Xi,\mathtt{d})$ is a complete metric space.

In the rest of this paper, we consider the \textsf{WDRO} problems satisfying the following two assumptions. The first assumption is on the smoothness and boundedness of $\theta$. Similar assumptions have been widely adopted in the machine learning field~\citep{DBLP:conf/icml/Zinkevich03,DBLP:conf/nips/WangGD21}. 
 
\begin{assumption}[Smoothness and Boundedness of $\theta$~\cite{understandingML}]
\label{ass:snb}
~
%We assume the following two properties hold.
\begin{enumerate}
    \item[(\romannumeral1)] ({\em{Boundedness}}) The feasible region  $\Theta$ of the parameter space for the \textsf{WDRO} problem (\ref{eq:WDRO}) is within a closed Euclidean ball $\mathbb{B}(\theta_{\mathtt{anc}},l_{\mathtt{p}})$ centered at some ``anchor'' point $\theta_{\mathtt{anc}}\in \mathbb{R}^d$ with radius $l_{\mathtt{p}}>0$;
    \item[(\romannumeral2)] ({\em{Lipschitz Smoothness}}\footnote{The methods proposed in this paper can be easily extended to other types of smoothness, e.g., gradient Lipschitz continuity.}) There exists a constant $L>0$, such that for any $\xi\in\Xi$ and any $\theta_1,\theta_2\in \mathbb{B}(\theta_{\mathtt{anc}},l_{\mathtt{p}})$, we have
$
|\ell\left(\theta_{1}, \xi\right)-\ell\left(\theta_{2}, \xi\right)| \leq L\left\|\theta_{1}-\theta_{2}\right\|_2.
$
\end{enumerate}
\end{assumption}
 The second assumption states that the loss function $\ell(\theta,\xi)$ is continuous and has a bounded growth rate on data $\xi$. The detailed growth rate functions are discussed in Section \ref{sec:app}.
\begin{assumption}[Continuity and Bounded Growth Rate of $\xi$]
\label{ass:bgr}
~
\begin{enumerate}
    \item[(\romannumeral1)]({\em{Continuity}}) The loss function $\ell(\theta,\cdot)$ is continuous for any $\theta\in\Theta$;
    \item [(\romannumeral2)]({\em{Bounded Growth Rate}}) There exists some positive continuous growth rate function $\mathtt{C}(\theta)$ and $\xi_0\in\Xi$ such that 
$$\ell(\theta,\xi)\leq \mathtt{C}(\theta)\left(1+\mathtt{d}^p\left(\xi,\xi_0\right)\right)$$ for any $\theta\in\Theta$ and any $\xi\in\Xi$. 
\end{enumerate}
\end{assumption}

Now we state the strong duality for the \textsf{WDRO}, which is an important property to guarantee the correctness of our dual coreset method.

\begin{proposition}[Strong duality~~\citep{articleDRO,blanchet2019quantifying,gao2016distributionally}]
\label{pro:duality}
For any upper semi-continuous $\ell(\theta,\cdot)$, any $\theta$ and any nominal distribution $\mathbb{P}$ with finite $p$-th moment, the worst-case risk satisfies
\begin{eqnarray}
\label{eq:dual-worst-risk}
R^\mathbb{P}_{\sigma, p}(\theta)=\inf _{\lambda \geq 0}\{ \lambda \sigma^{p}+H^\mathbb{P}(\theta,\lambda)\},
\end{eqnarray}

where $H^\mathbb{P}(\theta,\lambda)\coloneqq \mathbb{E}^{\mathbb{P}}\left[h(\theta,\lambda,\xi)\right]$ and $h(\theta,\lambda,\xi)=\sup\limits_{\zeta \in \Xi}\{ \ell(\theta,\zeta)-\lambda \mathtt{d}^p(\zeta,\xi)\}$ is the Moreau-Yosida regularization \citep{parikh2014proximal}. We use $\lambda^{\mathbb{P}}_*(\theta)$ to denote the $\lambda$ attaining the infimum in (\ref{eq:dual-worst-risk}).
\end{proposition}
\begin{remark}
\label{rmk-exist}
By the definition of $h$, for any given $\theta\in\Theta$, we can deduce that there always exists some $\lambda^\mathbb{P}_*(\theta)<\infty$ attaining the infimum of (\ref{eq:dual-worst-risk}).   
\end{remark}

For the sake of convenience, we abbreviate  $R_{\sigma,p}(\theta)=R^{\mathbb{P}_n}_{\sigma,p}(\theta)$, $H(\theta,\lambda)=H^{\mathbb{P}_n}(\theta,\lambda)$,   $\lambda_*(\theta)=\lambda^{\mathbb{P}_n}_*(\theta)$, and  $h_i(\theta,\lambda)=h(\theta,\lambda,\xi_i)$. 
We define the asymptotic growth rate function
$$\kappa(\theta)\coloneqq\limsup _{\mathtt{d}\left(\xi, \xi_0\right) \rightarrow \infty} \frac{\ell(\theta,\xi)-\ell\left(\theta,\xi_0\right)}{\mathtt{d}^{p}\left(\xi, \xi_0\right)}
$$
so as to conclude the continuity of $h_i(\cdot,\cdot)$ in the following two claims. Here $\xi_0$ is the point in Assumption \ref{ass:bgr} (\romannumeral2). %By convention we set $\kappa\equiv 0$ if $\Xi$ is bounded.}
\begin{claim}[Continuity of $h_i$ on $\theta$]
\label{cla:beta-continuous}
For each $i\in\{1,\dots,n\}$ and any fixed $\lambda\geq0$, we have
$$|h_i(\theta,\lambda)-h_i(\theta',\lambda)|\leq L\|\theta-\theta'\|_2,$$
for any $\theta,\theta'\in \Theta$  with $\kappa(\theta),\kappa(\theta')\leq \lambda$.
\end{claim}

\begin{claim}[Continuity of $h_i$ on $\lambda$]
\label{cla:lambda-continuous}
For each $i\in\{1,\dots,n\}$ and any fixed $\theta\in\mathbb{R}^d$, we have  
$$|h_i(\theta,\lambda)-h_i(\theta,\lambda')|\leq \max\{r_i^p(\theta,\lambda),r_i^p(\theta,\lambda')\}|\lambda-\lambda'|, ~\forall \lambda,\lambda'\geq \kappa(\theta),$$
 where $r_i(\theta,\lambda)\coloneqq\min\limits_{\zeta\in\Xi}\{\mathtt{d}(\zeta,\xi_i)\mid \ell(\theta,\zeta)-\lambda \mathtt{d}^p(\zeta,\xi_i)=h_i(\theta,\lambda)\}$ is the closest distance between $\xi_i$ and all the $\zeta$s that attain the supremum of $\ell(\theta,\zeta)-\lambda\mathtt{d}^p(\zeta,\xi_i)$ in $\Xi$.
%The value of $\kappa(\theta)$ does not depend on the choice of $\xi_0$.
\end{claim}
\begin{remark}
\label{rem:interval}
The reason that we let $\lambda\geq\kappa(\theta)$ in the above claims is that each $h_i(\theta,\lambda)$ goes to infinity if $\lambda<\kappa(\theta)$. Without loss of generality\footnote{It is possible that $h_i(\theta,\kappa(\theta))=\infty$, e.g., $\ell(\theta,\xi)$ is the loss function of ordinary linear regression. In this case, the argument in this paper still holds with  slight modification. }, we suppose $h_i(\theta,\kappa(\theta))<\infty$ in this paper. 
\end{remark}

%We formulate the WDRO-coreset problem of interest via the strong duality:
%\begin{eqnarray}
%\label{eq:finite-sum} H(\theta,\lambda)=\sum\limits_{i=1}^n\frac{1}{n}h_i(\theta,\lambda)
%\end{eqnarray}
%where $h_i(\theta,\lambda)\coloneqq \sup\limits_{\xi\in\Xi} \ell(\theta,\xi)-\lambda d^p(\xi,\xi_i)$.

%Note that $r_i(\theta)$ is determined by the nominal distribution $\nu$ since $\lambda^*$ is the minimizer of $\lambda \sigma^{p}-\int_{\Xi} \inf _{\xi \in \Xi}\left[\lambda d^{p}(\xi, \zeta)-\Psi(\xi)\right] \nu(d \zeta)$. This implies that our coreset $\tilde \nu$ induces some $\tilde r_i(\theta)\neq r_i(\theta)$.

\section{From Coresets to Dual Coresets}
\label{sec-relation}
In this section, we provide the concept of ``dual coreset'' and prove that it is sufficient to guarantee the correctness with respect to the \textsf{WDRO} coreset. First, we present the definition of the  dual coreset  via directly combining Proposition \ref{pro:duality} and Definition \ref{def:coreset}. Suppose $I$ is an interval depending on $\theta$ (we will discuss this assumption in  detail later). 

%Combining Proposition \ref{pro:duality} and Definition \ref{def:coreset}, we can immediately give the strong dual coreset definition.
 
\begin{definition}[Dual $\epsilon$-Coreset]
\label{def:dualcoreset}
A dual $\epsilon$-coreset for the \textsf{WDRO} problem (\ref{eq:WDRO}) is a sparse non-negative mass vector $W=[w_1,\dots,w_n]$ such that the total mass $\sum\limits_{i=1}^nw_i=1$ and 
\begin{eqnarray}
\label{eq:dualcoreset}
\tilde H(\theta,\lambda)\coloneqq\sum\limits_{i=1}^n w_ih_i(\theta,\lambda)\in (1\pm \epsilon) H(\theta,\lambda)
\end{eqnarray}
for any $\theta\in \Theta$ and $\lambda\in I$. 
%The query region $I\subseteq [\kappa(\theta),\infty)$ will be discussed in details later.
\end{definition}
 
\begin{remark}
 Note that we require the approximation guarantee holds not only for any $\theta\in\Theta$, but also for any $\lambda\in I$ in the above definition. This is also a key difference to the traditional coresets. 
%This change is crucial since $\lambda$ controls the transport plan among the whole data set. Thus we cannot write the objective function in a finite-sum formulation unless $\lambda$ is regarded as part of the ``parameters''.
\end{remark}

By the discussion in Remark \ref{rem:interval}, we know $I\subset [\kappa(\theta),\infty)$. If we directly let $I=[\kappa(\theta),\infty)$, the   dual coreset of Definition~\ref{def:dualcoreset} requires to approximate the queries from all $\lambda\geq\kappa(\theta)$, which is too \textbf{strong}  and can be even troublesome for the coreset construction. Below we show that a bounded $I$ is sufficient for guaranteeing a dual coreset to be a qualified \textsf{WDRO} coreset.

Given a non-negative mass vector $W=[w_1,\cdots,w_n]$, the corresponding weighted empirical distribution is $\mathbb{\tilde P}_n=\sum_{i=1}^nw_i\delta_{\xi_i}$. 
Recall that we define a parameter $\lambda^{\mathbb{P}}_*(\theta)$ for duality in Proposition~\ref{pro:duality}. 
Together with Assumption \ref{ass:bgr}, we show the boundedness of the $\lambda^{\mathbb{\tilde P}_n}_*(\theta)$ (abbreviated as $\tilde \lambda_*(\theta)$ for convenience) for  $\mathbb{\tilde P}_n$. Let $[n]=\{1, 2,\cdots, n\}$. The following result is a key to relax the requirement for the dual coreset in Definition~\ref{def:dualcoreset}. 

\begin{lemma}[Boundedness of $\tilde\lambda_*$]
\label{lem:boundedness}
Given the empirical distribution $\mathbb{P}_n=\frac{1}{n}\sum\limits_{i=1}^n\delta_{\xi_i}$, we define the value $\rho=\max\limits_{i\in[n]}\{\mathtt{d}(\xi_i,\xi_0)\}$ that is the largest distance from the data samples to $\xi_0$. Here $\xi_0$ is defined in Assumption \ref{ass:bgr} (\romannumeral2). For any  $\theta\in\Theta$  and any mass vector $W$,  the  $\tilde \lambda_*(\theta)$ of the corresponding weighted empirical distribution  $\mathbb{\tilde P}_n$ is \textbf{no larger than}  
\begin{eqnarray}
\label{eq:boundedness}
  \mathtt{C}(\theta)\cdot\left(2^{p-1}+\frac{1+2^{p-1}\rho^p}{\sigma^p}\right), \quad 
\end{eqnarray}
where $\mathtt{C}(\theta)$ is defined in Assumption \ref{ass:bgr} (\romannumeral2). We use $\tau(\theta)$ to denote this upper bound  $\mathtt{C}(\theta)\left(2^{p-1}+\frac{1+2^{p-1}\rho^p}{\sigma^p}\right)$.
\end{lemma}

\begin{remark}
\textbf{(\romannumeral1)} In practice, we usually normalize the dataset before training a machine learning model, which implies that $\rho$ is not large. 
\textbf{(\romannumeral2)} It is worth noting that the above lemma can help us to  compute an upper bound for $\lambda_*(\theta)$. For example, if letting $W=[\frac{1}{n},\dots,\frac{1}{n}]$, (\ref{eq:boundedness}) directly yields an upper bound.
%{\color{red}
%\textbf{(\romannumeral3)} For logistic regression with the loss function  $\ell(\theta,\xi)=\log(1+\exp(-y\theta^\top x))$ and $p=1$, 
%,$p=1$ and $\mathbb{X}$ is endowed with $L_1$ norm,
%\citet{li2019first} provided an upper bound $\lambda_*(\theta_*)\leq\frac{0.2785}{\sigma}$ with no dependence on the empirical distribution.
%}
\end{remark}

The following theorem shows that the query region $I=[\kappa(\theta),\tau(\theta)]$  is  sufficient  for obtaining a coreset of the \textsf{WDRO} problem (\ref{eq:WDRO}).

\begin{theorem}[Sufficiency of the bounded query region]
\label{the-dualsuff}
If we let query region $I=[\kappa(\theta),\tau(\theta)]$ in Definition \ref{def:dualcoreset}, the dual $\epsilon$-coreset defined in such way also satisfies the coreset of Definition \ref{def:coreset}. %Namely, a dual $\epsilon$-coreset is also an $\epsilon$-coreset for the \textsf{WDRO} problem (\ref{eq:WDRO}). 
\end{theorem}
%\begin{remark}
%Actually the dual $\epsilon$-coreset is equivalent to an original $\epsilon$-coreset. However, in this paper we only need the one-side sufficiency.
%\end{remark}
Therefore in the rest of this paper, we let $I=[\kappa(\theta),\tau(\theta)]$ in Definition \ref{def:dualcoreset}. Theorem~\ref{the-dualsuff} also implies the following corollary. So we can only focus on solving the dual  \textsf{WDRO} problem (\ref{eq:WDRO}) on the obtained dual $\epsilon$-coreset. 
\begin{corollary}
Given $\alpha\geq 1$, we suppose the parameter vector $\theta_0$ yields an $\alpha$-approximation obtained on the dual $\epsilon$-coreset. Then $\theta_0$ is also an $(\alpha\cdot\frac{1+\epsilon}{1-\epsilon})$-approximation of the original \textsf{WDRO} (\ref{eq:WDRO}).
\end{corollary}
To end this section, similar to $\mathbb{B}(\theta_{\mathtt{anc}},l_{\mathtt{p}})$ in Assumption \ref{ass:bgr} (\romannumeral1), we define an interval $[\lambda_{\mathtt{anc}}-l_{\mathtt{d}},\lambda_{\mathtt{anc}}+l_{\mathtt{d}}]$ centered at some ``anchor'' point $\lambda_{\mathtt{anc}}>0$ with radius $l_{\mathtt{d}}>0$. To ensure that $[\kappa(\theta),\tau(\theta)]$ is within the interval $[\lambda_{\mathtt{anc}}-l_{\mathtt{d}},\lambda_{\mathtt{anc}}+l_{\mathtt{d}}]$ for all $\theta\in\Theta$, we let $\lambda_{\mathtt{anc}}\coloneqq\max\limits_{\theta\in\Theta}\{\kappa(\theta_{\mathtt{anc}}),\frac{\tau(\theta)}{2}\}$  and $l_{\mathtt{d}}\coloneqq\lambda_{\mathtt{anc}}$.

\section{The Construction of Dual Coresets}
\label{sec:coreset}
Following the results of Section~\ref{sec-relation}, we show how to compute a qualified dual coreset in this section. 
Suppose we can evaluate the lower and upper bounds for each $h_i(\cdot,\cdot)$ with respect to  a given couple  $(\lambda_{\mathtt{anc}},\theta_{\mathtt{anc}})$, namely, we have 
$$a_i(\theta_{\mathtt{anc}},\lambda_{\mathtt{anc}})\leq h_i(\theta_{\mathtt{anc}},\lambda_{\mathtt{anc}}) \leq b_i(\theta_{\mathtt{anc}},\lambda_{\mathtt{anc}})$$
for $1\leq i\leq n$.
We defer the details for obtaining such upper and lower bounds for each application to Section~\ref{sec:app}.

\begin{algorithm}[tb]
	\caption{Dual $\epsilon$-Coreset Construction}
	\label{alg:dualcoreset}
	\begin{algorithmic}
		\STATE {\bfseries Input:}  The empirical distribution $\mathbb{P}_n=\frac{1}{n}\sum_{i=1}^n\delta_{\xi_i}$,
		the Lipschitz constant $L$, the ``anchors'' $\theta_{\mathtt{anc}}$ and $\lambda_{\mathtt{anc}}$, and corresponding radii $l_{\mathtt{p}}$ and $l_{\mathtt{d}}$;
		the parameter $\epsilon\in(0,1)$; lower bound oracle $a_i(\cdot,\cdot)$ and upper bound oracle $b_i(\cdot,\cdot)$ for $i\in[n]$. 
		 
		\begin{enumerate}
		\item Compute  $A=\frac{1}{n}\sum_{i=1}^na_i(\theta_{\mathtt{anc}},\lambda_{\mathtt{anc}})$ and  $B=\frac{1}{n}\sum_{i=1}^nb_i(\theta_{\mathtt{anc}},\lambda_{\mathtt{anc}})$.
		\item Let $N=\lceil\log n\rceil$; initialize $W=[0, 0, \cdots, 0]\in \mathbb{R}^n$.
		\item The dataset $\{\xi_1,\dots,\xi_n\}$ is partitioned into $(N+1)^2$ cells $\{C_{ij}| 0\leq i,j\leq N\}$ as (\ref{eq:grid}). 
		\item For each $C_{ij}\ne\emptyset$, $0\leq i,j\leq N$:
		\begin{enumerate}
		\item  take a sample $Q_{ij}$ from $ C_{ij}$ uniformly at random, where the size $|Q_{ij}|$ depends on the parameters $\epsilon$, $l_{\mathtt{p}}$, $l_{\mathtt{d}}$ and $L$ (the exact value will be discussed in our following analysis in Section~\ref{sec:analysis});
		\item  for each sample $\xi_k\in Q_{ij}$,  assign the mass of quantity $w_k=\frac{|C_{ij}|}{n|Q_{ij}|}$;
		\end{enumerate}
		\end{enumerate}

		\STATE {\bfseries Output:} the mass vector $W=[w_1, w_2, \cdots, w_n]$ as the dual $\epsilon$-coreset.
	\end{algorithmic}
\end{algorithm}

\subsection{The Construction Algorithm} 
\label{sec:contruction}
We show the dual $\epsilon$-coreset construction procedure in Algorithm~\ref{alg:dualcoreset}, where the high-level idea is based on the following grid sampling. 
%{\color{red} If only use the upper bounds or lower bounds to partition, what happens?}

%{\color{red} it is better to mimick our icml'21 paper, first describe the partition procedure, and then use an ``algorithm'' to state formally.}
\textbf{Grid sampling.} Let $N=\lceil \log n \rceil$. Given the anchor $(\theta_{\mathtt{anc}},\lambda_{\mathtt{anc}})$, we can conduct the partitions over the dataset based on the lower bounds $a_i(\theta_{\mathtt{anc}},\lambda_{\mathtt{anc}})$ and upper bounds $b_i(\theta_{\mathtt{anc}},\lambda_{\mathtt{anc}})$ separately.  Let $A = \frac{1}{n}\sum_{i=1}^{n} a_i(\theta_{\mathtt{anc}},\lambda_{\mathtt{anc}})$ and $B = \frac{1}{n}\sum_{i=1}^{n} b_i(\theta_{\mathtt{anc}},\lambda_{\mathtt{anc}})$. Then we have the following partitions. 
\begin{eqnarray}
\label{eq:lowerlayer1}
A_0&=&\big\{\xi_i\mid a_i(\theta_{\mathtt{anc}},\lambda_{\mathtt{anc}}) \le A\big\}, \label{for-layer1}\\
A_j&=&\big\{\xi_i\mid 2^{j-1}A < a_i(\theta_{\mathtt{anc}},\lambda_{\mathtt{anc}}) \le
2^{j}A\big\}, 1 \le j \le N.\label{eq:lowerlayer2}
 \end{eqnarray}	
\begin{eqnarray}
\label{eq:upperlayer1}
B_0&=&\big\{\xi_i\mid b_i(\theta_{\mathtt{anc}},\lambda_{\mathtt{anc}}) \le B\big\}, \label{for-layer1}\\
B_j&=&\big\{\xi_i\mid 2^{j-1}B < b_i(\theta_{\mathtt{anc}},\lambda_{\mathtt{anc}}) \le
2^{j}B\big\}, 1 \le j \le N.\label{eq:upperlayer2}
 \end{eqnarray}

We denote the lower bound and upper bound partitions as $\mathcal{A} = \{A_0, \cdots,A_N\}$ and $\mathcal{B} = \{B_0,\cdots,B_N\}$ respectively. Then, we compute the intersections over  $\mathcal{A}$ and $\mathcal{B}$ to generate the ``grid'':
\begin{eqnarray}
\label{eq:grid}
\mathcal{C} = \{ C_{ij}|C_{ij} =A_i \cap B_j, 0\leq i,j\leq N\}.
\end{eqnarray}
It is easy to see that $\mathcal{C}$ is a collection of disjoint ``cells''  and $\bigcup\limits_{i,j} C_{ij} = P$.  For each $\xi_k\in C_{ij}$, we have
\begin{eqnarray}
\label{eq:cell}
\mu_i\cdot 2^{i-1}A\leq h_k(\theta_{\mathtt{anc}},\lambda_{\mathtt{anc}})\leq 2^jB
\end{eqnarray}
 where  $\mu_i=0$ if $i=0$ and $\mu_i=1$ otherwise. 
Through the grid partition $\mathcal{C}$, we can  take a set of samples $Q_{ij}$ from $ C_{ij}$ uniformly at random, and assign the weight $\frac{|C_{ij}|}{n|Q_{ij}|}$  to each sample.
\begin{remark}
\label{rem:upper}
\textbf{(\romannumeral1)} The grid sampling is a variance reduction technique in the Monte-Carlo methods~\citep{gobet2016monte}, since the grid partition is also a stratification for $h_k(\theta_{\mathtt{anc}},\lambda_{\mathtt{anc}})$ as shown in (\ref{eq:cell}). 
If we consider only the upper bounds $b_i(\theta_{\mathtt{anc}},\lambda_{\mathtt{anc}})$ or the lower bounds $a_i(\theta_{\mathtt{anc}},\lambda_{\mathtt{anc}})$, the obtained  partition is not a valid stratification for $h_k(\theta_{\mathtt{anc}},\lambda_{\mathtt{anc}})$.
\textbf{(\romannumeral2)} If we can obtain the exact value of $h_i(\theta_{\mathtt{anc}},\lambda_{\mathtt{anc}})$, i.e.,  $a_i(\theta_{\mathtt{anc}},\lambda_{\mathtt{anc}})=b_i(\theta_{\mathtt{anc}},\lambda_{\mathtt{anc}})=h_i(\theta_{\mathtt{anc}},\lambda_{\mathtt{anc}})$, then the grid partition is exactly the spatial partition that was studied before~\cite{chen2009coresets,DBLP:conf/nips/WangGD21}.
%A canonical choice for $a_i(\theta_{\mathtt{anc}},\lambda_{\mathtt{anc}})$ is to set $a_i(\theta_{\mathtt{anc}})=\ell_i(\theta_{\mathtt{anc}})$.
%An alternative choice for the layer criterion quantity $B$ is $A+L_dR^p+C$~\cite{}. This is because $H(\theta,\lambda_{\mathtt{anc}})$ is non-increasing on $\lambda$, which implies $H(\theta_{\mathtt{anc}},\lambda_{\mathtt{anc}})\leq H(\lambda^*,\theta_{\mathtt{anc}})$ for $\lambda_{\mathtt{anc}}\geq \lambda^*$. Further more we have $\lambda^*\sigma^p+H(\lambda^*,\theta_{\mathtt{anc}})=R_{\sigma,p}(\theta_{\mathtt{anc}})\leq \mathbb{E}^{\mathbb{P}_n}[\ell(\theta_{\mathtt{anc}},\xi)]+L_d\sigma^p+C$. Overall, we have $H(\theta_{\mathtt{anc}},\lambda_{\mathtt{anc}})\leq A+L_d\sigma^p+C-\lambda^*\sigma^p\leq A+(L_d-\kappa)\sigma^p+C$.
%Thus we can pick $B\coloneqq \min\{\frac{1}{n}\sum_{i=1}^nb_i,A+L_d\sigma^p+C\}$.
\end{remark}

\subsection{Theoretical Analysis}
\label{sec:analysis}
In this section we analyze the  complexity of Algorithm \ref{alg:dualcoreset} in theory. Recall that we define  $r_i(\theta,\lambda)=\min\limits_{\zeta\in\Xi}\{\mathtt{d}(\zeta,\xi_i)\colon\ell(\theta,\zeta)-\lambda \mathtt{d}^p(\zeta,\xi_i)=h_i(\theta,\lambda)\}$ in Claim~\ref{cla:lambda-continuous}. 
The following theorem provides an asymptotic sample complexity of Algorithm \ref{alg:dualcoreset}. To state the theorem clearly, we define two notations $R\coloneqq\max\limits_{i\in[n]\atop \theta\in\Theta}\{r^p_i(\kappa(\theta),\theta)\}$ and $H\coloneqq\min\limits_{\theta\in\Theta\atop \lambda\in[\lambda_{\mathtt{anc}}-l_{\mathtt{d}},\lambda_{\mathtt{anc}}+l_{\mathtt{d}}]}H(\theta,\lambda)$.

%For simplicity, we use $R$  to denote the value $\max\limits_{ \theta\in\Theta} R(\theta)$.
\begin{theorem}
\label{the:wdrocoreset}
Set $|Q_{ij}|=\tilde{O}\left(\left(B\cdot\frac{B+L l_{\mathtt{p}}+Rl_{\mathtt{d}}}{AH}\right)^{2} \cdot \frac{d}{\epsilon^{2}}\right)$\footnote{$\tilde O(g)\coloneqq O(g\cdot\mathtt{polylog}(\frac{nLl_{\mathtt{p}} Rl_{\mathtt{d}}}{\epsilon H}))$} in the Algorithm \ref{alg:dualcoreset}. Then the returned $W$ is a qualified dual $\epsilon$-coreset with probability at least $1-\frac{1}{n}$. The construction time is $O(n\cdot \mathtt{time}_{ab})$ where  $\mathtt{time}_{ab}$ is the time complexity for computing the lower bound $a_i(\theta_{\mathtt{anc}},\lambda_{\mathtt{anc}})$ and the upper bound $b_i(\theta_{\mathtt{anc}},\lambda_{\mathtt{anc}})$ for each $h_i(\theta_{\mathtt{anc}},\lambda_{\mathtt{anc}})$.
\end{theorem}
\begin{remark}
Note that the value $H\geq \min_{\theta\in\Theta}\mathbb{E}^{\mathbb{P}_n}\ell(\theta,\xi)$, which should not be too small in practice since the loss function  $\ell(\cdot,\cdot)$ usually contains positive penalty terms. The value of $R$ will be discussed in Section~\ref{sec:app}.
\end{remark}
We show the sketched proof of Theorem \ref{the:wdrocoreset} below.
Based on the continuity of $h_i(\cdot,\cdot)$ and the Hoeffding's inequality~\citep{hoeffding1994probability}, for a fixed couple $(\theta,\lambda)$, we provide an upper bound on the sample complexity first. The bound ensures that the estimation for each cell $C_{ij}$ has a bounded deviation with  high probability.
%. The value of $\Delta_{ij}$ will be determined later.
\begin{lemma}
\label{lem:eachcell}
Let $\delta$ be a given positive number. We fix a couple $(\theta,\lambda)\in\mathbb{B}(\theta_{\mathtt{anc}},l_{\mathtt{p}})\times[\kappa(\theta),\tau(\theta)]$ and take a uniform sample $Q_{ij}$ from $C_{ij}$ with the sample size
\begin{eqnarray}
\label{eq:samplesize}
|Q_{ij}|=O\left((2^jB-\mu_i\cdot2^{i-1}A+2Ll_{\mathtt{p}}+2 Rl_{\mathtt{d}})^2\delta^{-2}\log\frac{1}{\eta}\right).
\end{eqnarray}
Then, we have the probability 
\begin{eqnarray}
\label{eq:concentration}
\mathtt{Prob}\left[\left|\frac{1}{|Q_{ij}|}\sum\limits_{\xi_k\in Q_{ij}}h_k(\theta,\lambda)-\frac{1}{|C_{ij}|}\sum\limits_{\xi_k\in C_{ij}}h_k(\theta,\lambda)\right|\geq \delta\right]\leq \eta.
\end{eqnarray}
\end{lemma}
We aggregate the deviations from all the cells to obtain the overall estimation error for the coreset.  
 To guarantee the approximation quality of  (\ref{eq:dualcoreset}), we need to design a sufficiently small value of the deviation $\delta$ for each cell $C_{ij}$ under our grid partition framework.  
 %Formally, we have the following lemma.
\begin{lemma}
\label{lem:delta}
In Lemma \ref{lem:eachcell}, we set the deviation  $\delta=\epsilon_1 (2^{j-1}+2^{i-1})A$ ~for $0\leq i,j\leq N$. Then we have 
\begin{eqnarray}
\label{eq:fixapproximation}
\mathtt{Prob}\left[|\tilde H(\theta,\lambda)-H(\theta,\lambda)|\leq 3\epsilon_1 H(\theta_{\mathtt{anc}},\lambda_{\mathtt{anc}})\right]\geq 1-(N+1)^2\eta.
\end{eqnarray}
\end{lemma}
To generalize the result of Lemma~\ref{lem:delta} to the whole feasible region $\mathbb{B}(\theta_{\mathtt{anc}},l_{\mathtt{p}})\times[0,2l_{\mathtt{d}}]$, we apply the discretization idea.  
Imagine to generate the axis-parallel grid with side length $\frac{\epsilon_3 l_{\mathtt{p}}}{\sqrt{d}}\times\epsilon_2l_{\mathtt{d}}$ inside   $\mathbb{B}(\theta_{\mathtt{anc}},l_{\mathtt{p}})\times[0,2l_{\mathtt{d}}]$; the parameters $\epsilon_2$ and $\epsilon_3$ are two small numbers that will be determined in our following analysis. For each grid cell we arbitrarily take a $(\theta,\lambda)$ as its representative point. 
Let $G$ be the set of
the selected representative points; it is easy to see  the cardinality $|G|=\frac{1}{\epsilon_2}\cdot O(\frac{1}{\epsilon_3^d}) $. Through taking the union bound over all $(\theta,\lambda)\in G$ for  (\ref{eq:fixapproximation}), we obtain the following Lemma \ref{lem:allbeta}.
\begin{lemma}
\label{lem:allbeta}
With probability at least $1-(N+1)^2|G|\eta$, we have
\begin{eqnarray}
\label{eq:representative}
|\tilde H(\theta,\lambda)-H(\theta,\lambda)|\leq 3\epsilon_1 H(\theta_{\mathtt{anc}},\lambda_{\mathtt{anc}})~ \text{for all}~ (\theta,\lambda)\in G.
\end{eqnarray}
 \end{lemma}
By using the above lemmas, we are ready to prove Theorem \ref{the:wdrocoreset}.

\begin{proof}(\textbf{of Theorem \ref{the:wdrocoreset}})
For any $(\theta,\lambda)\in \mathbb{B}(\lambda_{\mathtt{anc}},l_{\mathtt{d}})\times[\kappa(\theta),\tau(\theta)]$, we let $(\theta',\lambda')\in G$ be the representative point of the cell containing $(\theta,\lambda)$. Then we have $\|\theta-\theta'\|_2\leq \epsilon_3l_{\mathtt{p}}$ and $|\lambda'-\lambda|\leq \epsilon_2 l_{\mathtt{d}}$.  Without loss of generality, we assume $\lambda'\geq\lambda$. By using the triangle inequality,  we have
\begin{equation}
\begin{aligned}
& |h_k(\theta,\lambda)-h_k(\theta',\lambda')|\\
\leq & |h_k(\theta',\lambda')-h_k(\theta,\lambda')|+|h_k(\theta,\lambda)-h_k(\theta,\lambda')|\\
\leq & L\epsilon_3l_{\mathtt{p}}+ R\epsilon_2 l_{\mathtt{d}}. \quad\quad\text{(By Claim \ref{cla:lambda-continuous}, Claim \ref{cla:beta-continuous} and $\lambda'\geq\lambda\geq\kappa(\theta)$)}
\end{aligned}
\end{equation} 
The above inequality implies 
\begin{eqnarray}
\label{eq:1}
|H(\theta,\lambda)-H(\theta',\lambda')|\leq R\epsilon_2 l_{\mathtt{d}}+L\epsilon_3l_{\mathtt{p}}
\end{eqnarray}
and
\begin{eqnarray}
\label{eq:2}
|\tilde H(\theta,\lambda)-\tilde H(\theta',\lambda')|\leq R\epsilon_2 l_{\mathtt{d}}+L\epsilon_3l_{\mathtt{p}}.
\end{eqnarray}
Overall we have $|\tilde{H}(\theta,\lambda)-H(\theta,\lambda)|$
\begin{equation}
\begin{aligned}
%&|\tilde{H}(\theta,\lambda)-H(\theta,\lambda)| \\
\leq &\left|\tilde{H}(\theta,\lambda)-\tilde{H}\left(\theta',\lambda'\right)\right|+\left|\tilde{H}\left(\theta',\lambda'\right)-H\left(\theta',\lambda'\right)\right| 
+\left|H\left(\theta',\lambda'\right)-H\left(\theta,\lambda\right)\right| \\
\leq & 3 \epsilon_{1} H\left(\theta_{\mathtt{anc}},\lambda_{\mathtt{anc} }\right)+2 \times\left( R\epsilon_2 l_{\mathtt{d}}+L\epsilon_3l_{\mathtt{p}}\right) \quad (\text{By Lemma} ~\ref{lem:allbeta}, (\ref{eq:1}) ~\text{and}~ (\ref{eq:2}))
\end{aligned}
\end{equation}
By setting $\epsilon_1=\frac{H\epsilon}{9B}$, $\epsilon_2=\frac{ H\epsilon }{6l_{\mathtt{d}} R}$, $\epsilon_3=\frac{H\epsilon }{6Ll_{\mathtt{p}}}$ and $\eta=\frac{1}{n(N+1)^2|G|}$ and substituting them into $(\ref{eq:samplesize})$, we obtain the sample complexity as stated in Theorem \ref{the:wdrocoreset}. 
\end{proof}

\section{Applications}
\label{sec:app}
In this section, we show several \textsf{WDRO} problems that their complexities can be reduced by using 
our dual coreset method.

 \subsection{Binary Classification}
  For the binary classification, $\mathbb{Y}=\{-1,1\}$ and the loss function $\ell(\theta,\xi)=L(y\cdot\theta^\top x)$ where $L(\cdot)$ is a non-negative and non-increasing function. Let $\|\cdot\|_*$ be the dual norm of $\|\cdot\|$ on $\mathbb{R}^m$.
We consider the \textbf{Support Vector Machine (SVM)} and \textbf{logistic regression} problems. 
The SVM takes the hinge loss $L(z)=\max\{0,1-z\}$ and the logistic regression takes the logloss $L(z)=\log(1+\exp(-z))$.
If $\mathbb{X}=\mathbb{R}^m$ and $p=1$, by the result of \citet[Theorem 3.11]{shafieezadeh2019regularization}, for both of these two problems we have:
\begin{itemize}
    \item $R\leq\gamma$, $\kappa(\theta)=\mathtt{C}(\theta)=\|\theta\|_*$;
    \item $a_i(\theta,\lambda)=b_i(\theta,\lambda)=h_i(\theta,\lambda)=\max\{L(y_i\cdot\theta^\top x_i),L(-y_i\cdot\theta^\top x_i)-\lambda\gamma\}$ for any $\lambda\geq\kappa(\theta)$.
\end{itemize}

If $\mathbb{X}=[0,l]^m$ is a $m$-dimensional hypercube with side length $l>0$, $h_i(\theta,\lambda)$ in fact is the optimal objective value of a convex constrained programming problem. Therefore we \textbf{cannot} obtain the exact value of $h_i(\theta,\lambda)$ easily. To remedy this issue, we can invoke the upper and lower bounds of $h_i(\theta,\lambda)$ and conduct the grid sampling to efficiently construct the coreset. 

 \subsection{Regression}
 \label{sec:huber}
 For the regression problem, $\mathbb{Y}=\mathbb{R}$ and the loss function $\ell(\theta,\xi)=L(\theta^\top x-y)$, where $L(\cdot)$  is a non-negative function. Let $\|\cdot\|_*$ be the dual norm of $\|\cdot\|$ on $\mathbb{R}^{m+1}$. 
 We consider the \textbf{robust regression} problem that takes the Huber loss $L(z)=\frac{1}{2}z^2$ if $|z|\leq\delta$ and $L(z)=\delta\left(|z|-\frac{1}{2}\delta\right)$ otherwise for some $\delta\geq 0$. If $p=1$ and $\mathbb{X}=\mathbb{R}^m$, by the result of  \citet[Theorem 3.1]{shafieezadeh2019regularization}, we have 
\begin{itemize}
    \item $R=0, \kappa(\theta)=\mathtt{C}(\theta)=\delta\|(\theta,-1)\|_*$;
    \item $a_i(\theta,\lambda)=b_i(\theta,\lambda)=h_i(\theta,\lambda)=L(\theta^\top x_i-y_i)$ for any $\lambda\geq\kappa(\theta)$.
\end{itemize}

%\textbf{Ordinary linear regression}
%The ordinary linear regression invokes the squared error $L(z)=z^2$. If $p=2$ and $\|\cdot\|$ is the Euclidean norm on $\mathbb{R}^m$, then the \textsf{WDRO} of the ordinary linear regression is equivalent to
%\begin{equation}
%\label{eq:olr}
%\begin{array}{cll}
%\inf\limits_{\theta,\lambda,s_i} & \lambda\sigma^2+\frac{1}{n}\sum_{i=1}^ns_i & \\
%\text {s.t. } & \frac{(\theta^\top x_i-y_i)^2\|\theta\|_2^2}{\lambda-\|\theta\|_2^2}\leq s_i & i\in[n]\\
%& \frac{4(\theta^\top x_i-y_i)^2}{\lambda\gamma^2-4}\leq s_i & i\in[n]\\
%& \|\theta\|_2^2 \leq \lambda,\frac{4}{\gamma^2} \leq \lambda. 
%\end{array}
%\end{equation}

%\begin{itemize}
%    \item $\kappa(\theta)=\mathtt{C}(\theta)=\|(\theta,-1)\|_*$;
%    \item $h_i(\theta,\lambda)=\max\{\frac{(\theta^\top x_i-y_i)^2\|\theta\|_2^2}{\lambda-\|\theta\|_2^2},\frac{4(\theta^\top x_i-y_i)^2}{\lambda\gamma^2-4}\}$;
%    \item $r_i(\theta,\lambda)\leq\frac{|\theta^\top x_i-y_i|\|\theta\|_2}{\lambda-\|\theta\|_2^2}+\frac{2|\theta^\top x_i-y_i|\gamma}{\lambda\gamma^2-4}$.
%\end{itemize}

\section{Experiments}
\label{sec-exp}
Our experiments were conducted on a server equipped with 2.4GHZ Intel CPUs and 256GB main memory. The algorithms are implemented in Python. We use the MOSEK~\citep{mosek} to solve the tractable reformulations of \textsf{WDRO}s. Our code is available at \url{https://github.com/h305142/WDRO_coreset}.

\textbf{Compared methods.} 
We compare our dual coreset method \textsc{DualCore} with the uniform sampling approach \textsc{UniSamp}, the importance sampling approach \textsc{ImpSamp} \citep{tukan2020coresets}, the layer sampling approach \textsc{LayerSamp} \citep{huang2021novel}, and the approach that directly runs on whole training set \textsc{Whole}.

\textbf{Datasets.} 
We test the algorithms for the SVM and  logistic regression problems on two real datasets:
\textsc{Mnist}\citep{lecun2010mnist} and \textsc{Letter}\citep{chang2011libsvm}.
To simulate the scenarios  with contaminated datasets, we perform   poisoning attacks to the training set of \textsc{Letter}. Specifically,  we use the \textsc{Min-Max} attack from \citep{DBLP:journals/corr/abs-1811-00741} and \textsc{Alfa} attack from \citep{DBLP:journals/ijon/XiaoBNXER15}. We add the standard Gaussian noise $\mathcal{N}(0,1)$ to the training set of \textsc{Mnist} and randomly flip $10\%$ of the labels.
The dual coreset algorithm for the robust regression problem is evaluated on the real dataset \textsc{Appliances Energy}\citep{candanedo2017data}.

\textbf{Results.} 
Let $s$ and $n$ be the coreset size and the training set size, respectively. We set $c\coloneqq\frac{s}{n}$ to indicate the compression rate and fix the parameter $\gamma=7$ for all the instances (recall that $\gamma$ is used for defining the feature-label metric $\mathtt{d}(\xi_i,\xi_j)=\|x_i-x_j\|+\frac{\gamma}{2}|y_i-y_j|$). The experiments of  each instance were repeated by $50$ independent trials. We report the obtained worst-case risk $R_{\sigma,p}^{\mathbb{P}_n}(\theta_*)$ for each  method in table \ref{tab-lr}, \ref{tab-svm} and \ref{tab-huber}.

\begin{table}[!ht]
    \centering
    \begin{tabular}{l|l|l|l|l}
    \toprule[2pt]
        \textbf{$c$} & \textsc{UniSamp} & \textsc{ImpSamp} & \textsc{LayerSamp} & \textsc{DualCore} \\ \hline
        1\% & 0.72518$\pm$0.1268 & 0.70484$\pm$0.0915 & 0.70988$\pm$0.084 & \textbf{0.68933$\pm$0.0587} \\ 
        2\% & 0.64630$\pm$0.0344 & 0.65798$\pm$0.0447 & 0.63911$\pm$0.0305 & \textbf{0.63708$\pm$0.0257} \\ 
        3\% & 0.62709$\pm$0.0216 & 0.63015$\pm$0.0333 & \textbf{0.62381$\pm$0.0199} & 0.62546$\pm$0.0225 \\ 
        4\% & 0.62047$\pm$0.0176 & 0.6235$\pm$0.0183 & 0.61616$\pm$0.0143 & \textbf{0.61292$\pm$0.0149} \\ 
        5\% & 0.61338$\pm$0.0164 & 0.61524$\pm$0.0137 & 0.61013$\pm$0.0096 & \textbf{0.60986$\pm$0.0097} \\ 
        6\% & 0.60823$\pm$0.0084 & 0.61284$\pm$0.0131 & 0.60749$\pm$0.0119 & \textbf{0.60556$\pm$0.0092} \\ 
        7\% & 0.60716$\pm$0.0082 & 0.61198$\pm$0.0113 & 0.6059$\pm$0.0083 & \textbf{0.60381$\pm$0.0073} \\ 
        8\% & 0.60640$\pm$0.007 & 0.60936$\pm$0.0108 & 0.60376$\pm$0.0078 & \textbf{0.60238$\pm$0.0062} \\ 
        9\% & 0.60395$\pm$0.0066 & 0.60677$\pm$0.0086 & 0.60235$\pm$0.007 & \textbf{0.60056$\pm$0.0046} \\ 
        10\% & 0.60220$\pm$0.0069 & 0.60574$\pm$0.009 & \textbf{0.60007$\pm$0.0041} & 0.60113$\pm$0.0071 \\ \bottomrule[2pt]
    \end{tabular}
    \caption{Worst-case risk of logistic regression on \textsc{Letter}, \textsc{Whole}=0.59267, $\sigma=0.3$}
    \label{tab-lr}
\end{table}

\begin{table}[!ht]
    \centering
    \begin{tabular}{l|l|l|l|l}
    \toprule[2pt]
        $c$ & \textsc{UniSamp} & \textsc{ImpSamp} & \textsc{LayerSamp} & \textsc{DualCore} \\ \hline
        1\% & 0.68707$\pm$0.1094 & \textbf{0.66576$\pm$0.103} & 0.70577$\pm$0.1278 & 0.67866$\pm$0.1173 \\ 
        2\% & 0.59376$\pm$0.0565 & 0.60895$\pm$0.0683 & \textbf{0.58967$\pm$0.0529} & 0.59850$\pm$0.0548 \\ 
        3\% & 0.56860$\pm$0.036 & 0.57346$\pm$0.0453 & 0.56705$\pm$0.0377 & \textbf{0.56689$\pm$0.0347} \\ 
        4\% & 0.54429$\pm$0.0308 & 0.55050$\pm$0.0409 & 0.54366$\pm$0.0336 & \textbf{0.53634$\pm$0.0207} \\ 
        5\% & 0.53218$\pm$0.0234 & 0.54212$\pm$0.0295 & \textbf{0.52981$\pm$0.0182} & 0.53217$\pm$0.019 \\ 
        6\% & 0.5346$\pm$0.0248 & 0.53835$\pm$0.0288 & \textbf{0.52496$\pm$0.0177} & 0.52835$\pm$0.0184 \\ 
        7\% & 0.52784$\pm$0.0225 & 0.53388$\pm$0.0275 & 0.52039$\pm$0.015 & \textbf{0.52025$\pm$0.0147} \\ 
        8\% & 0.52246$\pm$0.019 & 0.51993$\pm$0.0119 & 0.51918$\pm$0.0126 & \textbf{0.51845$\pm$0.0116} \\ 
        9\% & 0.52025$\pm$0.0153 & 0.52402$\pm$0.0206 & 0.51289$\pm$0.0094 & \textbf{0.51196$\pm$0.0054} \\ 
        10\% & 0.51458$\pm$0.0083 & 0.51768$\pm$0.0166 & 0.51578$\pm$0.013 & \textbf{0.51066$\pm$0.0065} \\ \bottomrule[2pt]
    \end{tabular}
    \caption{Worst-case risk of SVM on \textsc{Letter}, \textsc{Whole}=0.49734, $\sigma=0.1$}
    \label{tab-svm}
\end{table}

\begin{table}[!ht]
    \centering
    \begin{tabular}{l|l|l|l}
    \toprule[2pt]
        $c$ & \textsc{UniSamp} & \textsc{ImpSamp} & \textsc{DualCore} \\ \hline
        1\% & 28.57655$\pm$0.0005 & 28.57649$\pm$0.0004 & \textbf{28.57627$\pm$0.0002} \\ 
        2\% & 28.57619$\pm$0.0001 & 28.57617$\pm$0.0001 & \textbf{28.57607$\pm$0.0001} \\ 
        3\% & 28.57609$\pm$0.0001 & 28.5761$\pm$0.0001 & \textbf{28.57598$\pm$0.0001} \\ 
        4\% & 28.57600$\pm$0.0001 & 28.57601$\pm$0.0001 & \textbf{28.57593$\pm$0.0001} \\ 
        5\% & 28.57595$\pm$0 & 28.57598$\pm$0.0001 & \textbf{28.57588$\pm$0} \\ 
        6\% & 28.57593$\pm$0 & 28.57594$\pm$0.0001 & \textbf{28.57587$\pm$0} \\ 
        7\% & 28.57591$\pm$0 & 28.57592$\pm$0.0001 & \textbf{28.57585$\pm$0} \\ 
        8\% & 28.57589$\pm$0 & 28.57589$\pm$0 & \textbf{28.57584$\pm$0} \\ 
        9\% & 28.57589$\pm$0 & 28.57589$\pm$0 & \textbf{28.57583$\pm$0} \\ 
        10\% & 28.57588$\pm$0 & 28.57588$\pm$0 & \textbf{28.57582$\pm$0} \\ \bottomrule[2pt]
    \end{tabular}
    \caption{Worst-case risk of Huber regression on \textsc{Appliances Energy}, \textsc{Whole}=28.57578, $\sigma=100$.
     \textsc{LayerSamp} coinsides with \textsc{DualCore} for Huber regression. This is because $h(\theta,\lambda,\xi_i)=\ell(\theta,\xi_i)$ for Huber regression (See Section \ref{sec:huber}).}
    \label{tab-huber}
\end{table}

%For the \textsf{WDRO} logistic regression and SVM problems, we also report the averaged test accuracy and the standard deviation . 

\section{Conclusion}
 In this paper, we consider reducing the high computational complexity of \textsf{WDRO} via the coreset method. We relate the coreset to its dual coreset by using the strong duality property of \textsf{WDRO}, and propose a novel grid sampling approach for the construction. To the best of our knowledge, our work is the first systematically study on the coreset of \textsf{WDRO} problems in theory.  We also implement our proposed coreset algorithm and conduct the experiments to evaluate its performance for several \textsf{WDRO} problems (including the applications mentioned in Section \ref{sec:app}). Following our work, there also exist several important problems deserving to study in future. For example, it is interesting to consider the coresets construction for other robust optimization models (e.g., adversarial training~\cite{goodfellow2018making}). 
 
\section{Acknowledgements}
 The research of this work was supported in part by National Key R$\&$D program of China through grant 2021YFA1000900 and the Provincial NSF of Anhui through grant 2208085MF163. We also want to thank the anonymous reviewers for their helpful comments.

%It is interesting to study the \textsf{WDRO} coreset for large models like deep neural networks, which may be utilized to deal with the catastrophic forgetting in the streaming setting \citep{borsos2020coresets}.

\bibliography{RobustCoreset}

\appendix

\section{Omitted Proofs}
\subsection{Proof of Remark \ref{rmk-exist}}
By the definition of $h$, we have $h(\theta,\lambda,\xi)\geq \ell(\theta,\xi)-\lambda \mathtt{d}^p(\xi,\xi)=\ell(\theta,\xi)$. Therefore $h(\theta,\lambda,\xi)$ is always nonnegative and so is $H^\mathbb{P}(\theta,\lambda)$. As a consequence, $\lambda\sigma^p+H^\mathbb{P}(\theta,\lambda)$ grows to infinity as $\lambda\rightarrow\infty$. Together with the fact that $\lambda\sigma^p+H^\mathbb{P}(\theta,\lambda)$ is continuous on $\lambda$~\citep[lemma 3 (\romannumeral2)]{gao2016distributionally},we can deduce that there always exists some $\lambda^\mathbb{P}_*(\theta)<\infty$ attaining the infimum of (\ref{eq:dual-worst-risk}) for any given $\theta$.
\subsection{Proof of Lemma \ref{lem:boundedness}}
\begin{lemma}[Boundedness of $\tilde\lambda_*$]
\label{lem:boundedness}
Given the empirical distribution $\mathbb{P}_n=\frac{1}{n}\sum\limits_{i=1}^n\delta_{\xi_i}$, we define the value $\rho=\max\limits_{i\in[n]}\{\mathtt{d}(\xi_i,\xi_0)\}$ that is the largest distance from the data samples to $\xi_0$. Here $\xi_0$ is defined in Assumption \ref{ass:bgr} (\romannumeral2). For any  $\theta\in\Theta$  and any mass vector $W$,  the  $\tilde \lambda_*(\theta)$ of the corresponding weighted empirical distribution  $\mathbb{\tilde P}_n$ is \textbf{no larger than}  
\begin{eqnarray}
\label{eq:boundedness}
  \mathtt{C}(\theta)\cdot\left(2^{p-1}+\frac{1+2^{p-1}\rho^p}{\sigma^p}\right), \quad 
\end{eqnarray}
where $\mathtt{C}(\theta)$ is defined in Assumption \ref{ass:bgr} (\romannumeral2). We use $\tau(\theta)$ to denote this upper bound  $\mathtt{C}(\theta)\left(2^{p-1}+\frac{1+2^{p-1}\rho^p}{\sigma^p}\right)$.
\end{lemma}

 \begin{proof}
First, by Assumption \ref{ass:bgr} we have 
\begin{eqnarray}
\ell(\theta,\zeta)-\lambda \mathtt{d}^p(\zeta,\xi)\leq  \mathtt{C}(\theta)\left(1+\mathtt{d}^p\left(\zeta,\xi_0\right)\right)-\lambda \mathtt{d}^p(\zeta,\xi), \forall \zeta\in \Xi.
\end{eqnarray}
 If letting $\lambda=2^{p-1}\mathtt{C}(\theta)$,  the above inequality  yields
\begin{eqnarray}
\label{eq:tmp}
 \ell(\theta,\zeta)-2\mathtt{C}(\theta)\mathtt{d}^p(\zeta,\xi)\leq \mathtt{C}(\theta)(1+\mathtt{d}^p(\zeta,\xi_0))-2^{p-1}\mathtt{C}(\theta)\mathtt{d}^p(\zeta,\xi), \forall \zeta\in \Xi.
\end{eqnarray}

Further, combining the triangle inequality and Jensen's inequality~\citep{jensen1906fonctions}, we know that 
\begin{eqnarray}
\label{eq:jensentriangle}
\mathtt{d}^p(\zeta,\xi_0)\leq 2^{p-1}\mathtt{d}^p(\zeta,\xi)+2^{p-1}\mathtt{d}^p(\xi,\xi_0).
\end{eqnarray}

We then substitute (\ref{eq:jensentriangle}) into (\ref{eq:tmp}),  and based on the definition of $h(\theta,\lambda,\xi)$  we have 
\begin{eqnarray}
h(\theta,2^{p-1}\mathtt{C}(\theta),\xi)\leq \mathtt{C}(\theta)(1+2^{p-1}\mathtt{d}^p(\xi,\xi_0)).
\end{eqnarray}

Note that $\tilde H(\theta,\lambda)=\mathbb{E}^{\mathbb{\tilde P}_n}[h(\theta,\lambda,\xi)]$ and $\mathtt{d}(\xi_i,\xi_0)\leq\rho~ \text{for all} ~i\in[n]$. So we have 
\begin{eqnarray}
\tilde H(\theta,2^{p-1}\mathtt{C}(\theta))\leq \mathtt{C}(\theta)(1+2^{p-1}\rho^p)
\end{eqnarray}
for any mass vector $W$.  
Therefore we obtain 
\begin{eqnarray}
\inf\limits_{\lambda\geq 0}\{\lambda\sigma^p+\tilde H(\theta,\lambda)\}\leq 2^{p-1}\mathtt{C}(\theta)\sigma^p+\mathtt{C}(\theta)(1+2^{p-1}\rho^p),
\end{eqnarray}
which implies that 
\begin{eqnarray}
\tilde\lambda^*(\theta)\leq \mathtt{C}(\theta)\left(2^{p-1}+\frac{1+2^{p-1}\rho^p}{\sigma^p}\right)
\end{eqnarray}
 for any $\theta\in\Theta$ and any mass vector $W$.
\end{proof}

\subsection{Proof of Theorem \ref{the-dualsuff}}
\label{app:boundedness}
\begin{proof}
Suppose $W=[w_1,\dots,w_n]$ satisfies Definition \ref{def:dualcoreset}. By Lemma \ref{lem:boundedness},  we have $\lambda_*(\theta),\tilde \lambda_*(\theta)\in[\kappa(\theta),\tau(\theta)]$, which implies that 
\begin{eqnarray}
\tilde H(\theta,\lambda_*(\theta))\in (1\pm\epsilon)H(\theta,\lambda_*(\theta)).
\end{eqnarray}
Then for any fixed $\theta\in\Theta$, we have both
\begin{eqnarray}
\label{eq:uppersuf}
\lambda_*\sigma^p+H(\theta,\lambda_*)\geq \frac{1}{1+\epsilon}(\lambda_*\sigma^p+\tilde H(\theta,\lambda_*))\geq \frac{1}{1+\epsilon}(\tilde \lambda_*\sigma^p+\tilde H(\theta,\tilde\lambda_*))
\end{eqnarray}
and 
\begin{eqnarray}
\label{eq:lowersuf}
\lambda_*\sigma^p+H(\theta,\lambda_*)\leq\tilde\lambda_*\sigma^p+H(\theta,\tilde\lambda_*)\leq \frac{1}{1-\epsilon}(\tilde \lambda_*\sigma^p+\tilde H(\theta,\tilde\lambda_*)).
\end{eqnarray}

From  proposition \ref{pro:duality} we have $R_{\sigma,p}(\theta)=H(\theta,\lambda_*)+\lambda_*\sigma^p$ and $\tilde R_{\sigma,p}(\theta)=\tilde H(\theta,\tilde\lambda_*)+\tilde\lambda_*\sigma^p$. Together with (\ref{eq:uppersuf}) and (\ref{eq:lowersuf}),
they imply  
\begin{eqnarray}
\tilde R_{\sigma,p}(\theta)\in (1\pm\epsilon)R_{\sigma,p}(\theta).
\end{eqnarray}
Therefore, $W$ is a qualified coreset satisfying Definition \ref{def:coreset}.
\end{proof}
\subsection{Proof of Claim \ref{cla:beta-continuous}}
\label{app:beta}
\begin{proof}
Denote $z_1=\arg\max\limits_{z\in\Xi}\{\ell(\theta_1,z)-\lambda \mathtt{d}^p(z,\xi)\}$ and $z_2=\arg\max\limits_{z\in\Xi}\{\ell(\theta_2,z)-\lambda \mathtt{d}^p(z,\xi)\}$. Then $h(\theta_1,\lambda,\xi)=\ell(\theta_1,z_1)-\lambda \mathtt{d}^p(z_1,\xi)$ and $h(\theta_2,\lambda,\xi)=\ell(\theta_2,z_2)-\lambda \mathtt{d}^p(z_2,\xi)$.

By the definitions of $z_1$ and $z_2$, we have 
\begin{equation}
\label{eq:maximization}
\left.\begin{array}{l}
\ell(\theta_1,z_1)-\lambda \mathtt{d}^p(z_1,\xi) \geq \ell(\theta_1,z_2)-\lambda \mathtt{d}^p(z_2,\xi); \\
\ell(\theta_2,z_1)-\lambda \mathtt{d}^p(z_1,\xi) \leq \ell(\theta_2,z_2)-\lambda \mathtt{d}^p(z_2,\xi).
\end{array}\right\}
\end{equation}
By Assumption \ref{ass:snb} \romannumeral2, we have
\begin{equation}
\label{eq:beta_continuous}
\left.\begin{array}{l}
|\ell(\theta_1,z_1)- \ell(\theta_2,z_1)| \leq L\|\theta_1-\theta_2\|_2 ; \\
|\ell(\theta_1,z_2)- \ell(\theta_2,z_2)| \leq L\|\theta_1-\theta_2\|_2.
\end{array}\right\}
\end{equation}
Combining (\ref{eq:beta_continuous}) and (\ref{eq:maximization}), we have 
\begin{equation}
\label{eq:approximation}
\left.\begin{array}{l}
h(\theta_1,\lambda,\xi)-L\|\theta_1-\theta_2\|_2\overset{(\ref{eq:beta_continuous})}{\leq} \ell(\theta_2,z_1)-\lambda\mathrm{d}^p(z_1,\xi)\overset{(\ref{eq:maximization})}{\leq} h(\theta_2,\lambda,\xi); \\
 h(\theta_2,\lambda,\xi)-L\|\theta_1-\theta_2\|_2\overset{(\ref{eq:beta_continuous})}{\leq}\ell(\theta_1,z_2)-\lambda\mathrm{d}^p(z_2,\xi)\overset{(\ref{eq:maximization})}{\leq} h(\theta_1,\lambda,\xi).
\end{array}\right\}
\end{equation}
That is , $$|h(\theta_1,\lambda,\xi)-h(\theta_2,\lambda,\xi)|\leq L\|\theta_1-\theta_2\|_2 $$
\end{proof}
\subsection{Proof of Claim \ref{cla:lambda-continuous}}
\label{app:lambda}
\begin{proof}
From \cite[lemma 3 (\romannumeral2)]{gao2016distributionally}, we know that $h(\theta,\lambda,\xi)$ is convex and non-increasing in $\lambda$. 
Define 
$$\underline{D}(\theta,\lambda,\xi):=\liminf _{\delta \downarrow 0}\left\{\mathtt{d}(\xi, \zeta): \ell(\theta,\zeta)-\lambda\mathtt{d}^p(\xi,\zeta) \geq h(\theta,\lambda,\xi)-\delta\right\}.$$ Further, \citet[lemma 3 (\romannumeral4)]{gao2016distributionally} showed that $\underline{D}^p(\theta,\lambda,\xi)$ is a subderivative on $\lambda$ for 
 $h(\theta,\lambda,\xi)$. Therefore, from the convexity of $h_i(\theta,\cdot)$, we know 
 \begin{eqnarray}
 |h_i(\theta,\lambda)-h_i(\theta,\lambda')|\leq \max\{\underline D^p(\theta,\lambda,\xi_i),\underline D^p(\theta,\lambda',\xi_i)\}|\lambda-\lambda'|, \forall \lambda,\lambda'\geq\kappa(\theta).
 \end{eqnarray}
 
From Assumption \ref{ass:bgr} (\romannumeral1), we know that $\ell(\theta,\zeta)-\lambda\mathtt{d}^p(\xi,\zeta)$ is continuous in $\zeta$. 

If $\{\zeta\in\Xi\colon\ell(\theta,\zeta)-\lambda \mathtt{d}^p(\zeta,\xi_i)=h_i(\theta,\lambda)\}=\emptyset$,  by the continuity of $\ell(\theta,\cdot)-\lambda\mathtt{d}^p(\xi,\cdot)$ and the  mathematical analysis, we have $\underline D(\theta,\lambda,\xi_i)=\infty$. We set $r_i(\theta,\lambda)=\infty$ in this case.

If the set $\{\zeta\in\Xi\colon\ell(\theta,\zeta)-\lambda \mathtt{d}^p(\zeta,\xi_i)=h_i(\theta,\lambda)\}$ is non-empty, by the definition of $r_i(\theta,\lambda)$ and $\underline D(\theta,\lambda,\xi_i)$, we know that $r_i(\theta,\lambda)\geq \underline D(\theta,\lambda,\xi_i)$, which completes the proof.
 
%By the definition of $r_i(\theta,\lambda)$ and $\underline D(\theta,\lambda,\xi_i)$, we know that $r_i(\theta,\lambda)\geq \underline D(\theta,\lambda,\xi_i)$  as long as the set $\{\zeta\in\Xi\colon\ell(\theta,\zeta)-\lambda \mathtt{d}^p(\zeta,\xi_i)=h_i(\theta,\lambda)\}$ is non-empty.
 
%From Assumption \ref{ass:bgr} (\romannumeral1), we know that $\ell(\theta,\zeta)-\lambda\mathtt{d}^p(\xi,\zeta)$ is continuous in $\zeta$, which implies that the set $\{\zeta\in\Xi\colon\ell(\theta,\zeta)-\lambda \mathtt{d}^p(\zeta,\xi_i)=h_i(\theta,\lambda)\}$ is non-empty as long as $h_i(\theta,\lambda)<\infty$. 

%By the knowledge of mathematical analysis, we can obtain the limit inferior $\underline D(\theta,\lambda,\xi_i)=r_i(\theta,\lambda)$, which completes the proof.
 \end{proof}
 \subsection{Proof of Theorem 2}
 We show the proof of Theorem \ref{the:wdrocoreset} below.
Based on the continuity of $h_i(\cdot,\cdot)$ and the Hoeffding's inequality~\citep{hoeffding1994probability}, for a fixed couple $(\theta,\lambda)$, we provide an upper bound on the sample complexity first. The bound ensures that the estimation for each cell $C_{ij}$ has a bounded deviation with  high probability.
%. The value of $\Delta_{ij}$ will be determined later.
\begin{lemma}
\label{lem:eachcell}
Let $\delta$ be a given positive number. We fix a couple $(\theta,\lambda)\in\mathbb{B}(\theta_{\mathtt{anc}},l_{\mathtt{p}})\times[\kappa(\theta),\tau(\theta)]$ and take a uniform sample $Q_{ij}$ from $C_{ij}$ with the sample size
\begin{eqnarray}
|Q_{ij}|=O\left((2^jB-\mu_i\cdot2^{i-1}A+2Ll_{\mathtt{p}}+2 Rl_{\mathtt{d}})^2\delta^{-2}\log\frac{1}{\eta}\right).
\end{eqnarray}
Then, we have the probability 
\begin{eqnarray}
\label{eq:concentration}
\mathtt{Prob}\left[\left|\frac{1}{|Q_{ij}|}\sum\limits_{\xi_k\in Q_{ij}}h_k(\theta,\lambda)-\frac{1}{|C_{ij}|}\sum\limits_{\xi_k\in C_{ij}}h_k(\theta,\lambda)\right|\geq \delta\right]\leq \eta.
\end{eqnarray}
\end{lemma}
\begin{proof}
For any fixed $0\leq i,j\leq N$ , we regard $h_k(\theta,\lambda)$ as an independent random variable for each $\xi_k\in Q_{ij}$. 
We consider the following two cases:  (\romannumeral 1) $\kappa(\theta_{\mathtt{anc}})\leq\kappa(\theta)$ and (\romannumeral2) $\kappa(\theta_{\mathtt{anc}})>\kappa(\theta)$.

For case (\romannumeral1), we have $\lambda\geq\kappa(\theta)\geq\kappa(\theta_{\mathtt{anc}})$ and $\lambda_{\mathtt{anc}}\geq\kappa(\theta_{\mathtt{anc}})$. Together with Claim \ref{cla:lambda-continuous} and (\ref{eq:cell}), we have 
\begin{eqnarray}
\mu_i\cdot 2^{i-1}A-Rl_{\mathtt{d}}\leq h_k(\theta_{\mathtt{anc}},\lambda)\leq 2^jB+Rl_{\mathtt{d}}.
\end{eqnarray}

By using Claim \ref{cla:beta-continuous}, we further obtain the following upper and lower bounds for $h_k(\theta,\lambda)$: 
\begin{equation}
\label{eq:cellbound}
\left.\begin{array}{l}
h_{k}(\theta,\lambda) \leq \quad ~~~~2^{j} B+Rl_{\mathtt{d}}+L l_{\mathtt{p}}; \\
h_{k}(\theta,\lambda) \geq \mu_i 2^{i-1} A-Rl_{\mathtt{d}}-L l_{\mathtt{p}}.
\end{array}\right\}
\end{equation}

For case (\romannumeral2), we have $\lambda_{\mathtt{anc}}\geq\kappa(\theta_{\mathtt{anc}})>\kappa(\theta)$ and $\lambda\geq\kappa(\theta)$. We apply Claim \ref{cla:beta-continuous} and  have 
\begin{eqnarray}
\mu_i\cdot 2^{i-1}A-Ll_{\mathtt{p}}\leq h_k(\theta,\lambda_{\mathtt{anc}})\leq 2^jB+Ll_{\mathtt{p}}.
\end{eqnarray}

Together with Claim \ref{cla:lambda-continuous} we can achieve the same lower and upper bounds as (\ref{eq:cellbound}). 
%Hence both cases admit (\ref{eq:cellbound}).

Let the sample size $|Q_{ij}|=O((2^jB-\mu_i2^{i-1}A+2Ll_{\mathtt{p}}+2Rl_{\mathtt{d}})^2\delta^{-2}\log\frac{1}{\eta})$. Through the Hoeffding's inequality ~\citep{hoeffding1994probability}, we have
$$\mathbb{P}\left[\left|\frac{1}{|Q_{ij}|}\sum\limits_{\xi_k\in Q_{ij}}h_k(\theta,\lambda)-\frac{1}{|C_{ij}|}\sum\limits_{\xi_k\in C_{ij}}h_k(\theta,\lambda)\right|\geq \delta\right]\leq \eta.$$
\end{proof}
We aggregate the deviations from all the cells to obtain a total error on the coreset. 
 To guarantee the approximation in (\ref{eq:dualcoreset}), we need to design a sufficiently small value of the deviation $\delta$ for each cell $C_{ij}$ based on the grid partition structure. 
 %Formally, we have the following lemma.
\begin{lemma}
\label{lem:delta}
In Lemma \ref{lem:eachcell}, we set the deviation  $\delta=\epsilon_1 (2^{j-1}+2^{i-1})A$ ~for $0\leq i,j\leq N$. Then we have 
\begin{eqnarray}
\label{eq:fixapproximation}
\mathtt{Prob}\left[|\tilde H(\theta,\lambda)-H(\theta,\lambda)|\leq 3\epsilon_1 H(\theta_{\mathtt{anc}},\lambda_{\mathtt{anc}})\right]\geq 1-(N+1)^2\eta.
\end{eqnarray}
\end{lemma}

\begin{proof}
Based on Lemma \ref{lem:eachcell}, we have 
\begin{eqnarray}
\left|\frac{|C_{ij}|}{|Q_{ij}|}\sum\limits_{\xi_k\in Q_{ij}}h_k(\theta,\lambda)-\sum\limits_{\xi_k\in C_{ij}}h_k(\theta,\lambda)\right|\leq |C_{ij}|\cdot \epsilon_1 (2^{j-1}+2^{i-1})A
\end{eqnarray}
with probability at least $1-\eta$. 
Through taking a union bound over all the cells, with probability at least $1-(N+1)^2\eta$, we have
\begin{eqnarray}
n|\tilde H(\theta,\lambda)-H(\theta,\lambda)|&=&|\sum\limits_{i,j}\sum\limits_{\xi_k\in Q_{ij}}\frac{|C_{ij}|}{|Q_{ij}|}h_k(\theta,\lambda)-\sum\limits_{i,j}\sum\limits_{\xi_k\in C_{ij}}h_k(\theta,\lambda)| \nonumber \\
&\leq& \sum\limits_{i,j}|C_{ij}|\epsilon_1 (2^{j-1}+2^{i-1})A \nonumber\\
&\leq& \sum\limits_{i,j}|C_{ij}|\epsilon_1 (2^{j-1}+2^{i-1})H(\theta_{\mathtt{anc}},\lambda_{\mathtt{anc}}). \label{eq:deviation}
\end{eqnarray}
We also need the following claim to proceed our proof.
\begin{claim}
\label{cla:delta}
$\sum\limits_{i,j}|C_{ij}|2^i\leq 3n$ and $\sum\limits_{i,j}|C_{ij}|2^j\leq 3n$
\end{claim}

Based on Claim \ref{cla:delta}, we can rewrite (\ref{eq:deviation}) as 
$$n|\tilde H(\theta,\lambda)-H(\theta,\lambda)|\leq 3n\epsilon_1H(\theta_{\mathtt{anc}},\lambda_{\mathtt{anc}}).$$
So the statement of Lemma~\ref{lem:delta} is true. 
\end{proof}

\begin{proof}\textbf{(of Claim~\ref{cla:delta})}
By the definition of $C_{ij}$, we have
$$\begin{array}{ll}2^{i} A=A, & \text { if } i=0; \\ 2^{i} A \leq 2 a_{k}\left(\theta_{\mathtt{anc}}\right), \forall\xi_k \in C_{ij}, & \text { if } i \geq 1.\end{array}$$
So we have $2^iA\leq A+2a_k(\theta_{\mathtt{anc}},\lambda_{\mathtt{anc}})$ for all $0\leq i\leq N$ and $\xi_k\in C_{ij}$. Overall, we have

$$
\begin{aligned}
\sum_{i,j=0}^{N}\left|C_{ij}\right| 2^{i} A &=\sum_{i,j=0}^{N} \sum_{\xi_k \in C_{ij}} 2^{i} A \\
& \leq \sum_{i,j=0}^{N} \sum_{\xi_k \in C_{ij}}\left(2 a_{k}\left(\theta_{\mathtt{anc}},\lambda_{\mathtt{anc}}\right)+A\right) \\
&=2 n A+n A=3 n A.
\end{aligned}
$$
Thus $\sum\limits_{i,j}|C_{ij}|2^i\leq 3n$, and we can 
prove $\sum\limits_{i,j}|C_{ij}|2^j\leq 3n$ via the same manner.
\end{proof}

\section{Omitted Discussions}
\subsection{Discussions on SVM in the Hypercube}
We discuss more details for \textbf{SVM in the hypercube}. 
Suppose $\mathbb{X}=[0,l]^d$  is a $d$-dimensional hypercube and $p=1$, then by \citep{shafieezadeh2019regularization}  and the strong duality of the linear programming, we know that the \textsf{WDRO} of SVM is equivalent to  
\begin{equation}
\label{eq:hypercubeWDRO}
\begin{array}{cll}
\inf\limits_{\theta, \lambda, s_{i},p_{i}^{+}, p_{i}^{-},\atop z_i^+,z_i^-} & \lambda \sigma+\frac{1}{n} \sum_{i=1}^{n} s_{i} & \\
\text {s.t. } & 1+l\cdot e^\top z_i^++x_i^\top p_i^+ \leq s_{i} &  \\
& 1+l\cdot e^\top z_i^-+x_i^\top p_i^--\gamma \lambda \leq s_{i} & \\
& -y_i\theta-p_i^+ \leq z_i^+, \vec{0}\leq z_i^+ & \\
& y_i\theta-p_i^-\leq z_i^-, \vec{0}\leq z_i^- & \\
& \left\|p_{i}^{+}\right\|_{*} \leq \lambda,\left\|p_{i}^{-}\right\|_{*} \leq \lambda, 0\leq s_i & i \in[n]
\end{array}
\end{equation}
where $e=[1,\dots,1]\in\mathbb{R}^d$. 
%For two vectors $p=[p^{(1)},\dots,p^{(d)}]$ and $q=[q^{(1)},\dots,q^{(d)}]$, the inequality $p\leq q$ means that $p^{(i)}\leq q^{(i)}$ for all $i\in[d]$.
Hence $h_i(\theta,\lambda)$ is equivalent to 
\begin{equation}
\label{eq:hypercubeh_i}
\begin{array}{cll}
\inf\limits_{p_{i}^{+}, p_{i}^{-},\atop z_i^+,z_i^-} & \max\{0,1+l\cdot e^\top z_i^++x_i^\top p_i^+,1+l\cdot e^\top z_i^-+x_i^\top p_i^--\gamma\lambda\} & \\
\text {s.t. } & -y_i\theta-p_i^+ \leq z_i^+, \vec{0}\leq z_i^+ & \\
& y_i\theta-p_i^-\leq z_i^-, \vec{0}\leq z_i^- & \\
& \left\|p_{i}^{+}\right\|_{*} \leq \lambda,\left\|p_{i}^{-}\right\|_{*} \leq \lambda & i \in[n].
\end{array}
\end{equation}
 
 In this task we have 
\begin{itemize}
    \item $\kappa\equiv0$, $\mathtt{C}(\theta)=\|\theta\|_*$ and $R\leq  \gamma+l\cdot d^{\frac{1}{2}}$;
    \item $h_i(\theta,\lambda)$ is the optimal value of a constrained convex programming in (\ref{eq:hypercubeh_i});
    \item For any $z_i^+,z_i^-,p_i^+,p_i^-$ satisfying the constraints in (\ref{eq:hypercubeh_i}), $\max\{0,1+l\cdot e^\top z_i^++x_i^\top p_i^+,1+l\cdot e^\top z_i^-+x_i^\top p_i^--\gamma \lambda\}$ is an upper bound for $h_i(\theta,\lambda)$ and thus can be viewed as $b_i(\theta,\lambda)$. Here we propose a simple strategy for determining the values for these variables. 
    
    If $\|\theta\|_*>\lambda$, set $$p_i^+=-\frac{\lambda y_i\theta}{\|\theta\|_*},p_i^-=\frac{\lambda y_i\theta}{\|\theta\|_*},\atop z_i^+=\max\{(-y_i+\frac{\lambda y_i}{\|\theta\|_*})\theta,\vec{0}\},z_i^-=\max\{(y_i-\frac{\lambda y_i}{\|\theta\|_*})\theta,\vec{0}\};$$
    otherwise, set
    $$p_i^+=- y_i\theta,p_i^-= y_i\theta,\atop z_i^+=\vec{0},z_i^-=\vec{0}.$$
    \item $a_i(\theta,\lambda)=\ell(\theta,\xi_i)$.
\end{itemize}

\subsection{Discussions on Hyperparameters}
\subsubsection{Wasserstein ball radius $\sigma$}
Recall that we let $I=[\kappa(\theta),\tau(\theta)]$ in the dual coreset definition \ref{def:dualcoreset}, where $\tau(\theta,\sigma)=\mathtt{C}(\theta)\cdot\left(2^{p-1}+\frac{1+2^{p-1}\rho^p}{\sigma^p}\right)$ (see Lemma \ref{lem:boundedness}). Therefore our dual coreset is actually associated with the $\sigma$, the radius of the Wasserstein ball. We can easily observe that our dual coreset is ``monotonic'' on hyperparameter $\sigma$ and we summarize this ``monotonic'' property in the following corollary.
\begin{corollary}[Monotonic property]
\label{cor:monotonic}
Suppose $0<\sigma_1<\sigma_2$. Suppose $W$ is a dual $\epsilon$-coreset satisfying the definition \ref{def:dualcoreset} with $I(\theta,\sigma_1)=[\kappa(\theta),\tau(\theta,\sigma_1)]$. Then we have that $W$ is also a dual $\epsilon$-coreset satisfying the definition \ref{def:dualcoreset} with $I(\theta,\sigma_2)=[\kappa(\theta),\tau(\theta,\sigma_2)]$.
\end{corollary}
\begin{proof}
The proof is straightforward. Noting that $\tau(\theta,\sigma)$ is decreasing on $\sigma$ over $(0,+\infty)$, we have $I(\theta,\sigma_2)\subset I(\theta,\sigma_1)$, which implies the corollary.
\end{proof}
\subsubsection{Feature-label metric parameter $\gamma$}
Recall that $\gamma$ is used for defining the feature-label metric $\mathtt{d}(\xi_i,\xi_j)=\|x_i-x_j\|+\frac{\gamma}{2}|y_i-y_j|$. We consider two extreme conditions for the value of $\gamma$ in Wasserstein distributionally robust logistic regression as an example.

\begin{itemize}
    \item The value of $\gamma$ is arbitrary close to 0. In this case, a feature $x$ can be equipped with an arbitrary label $y$, which makes the Wasserstein distributionally robust logistic regression problem trivial. As a consequence, the optimal $\theta$ is always equal to $\vec 0$ with no dependence on the training data. Therefore in this case, the coreset of WDRO can be any subset of the original training set.
    \item The value of $\gamma$ is arbitrary large. In this case, the Wasserstein distributionally robust logistic regression reduces to the logistic regression with $\|\cdot\|_*$ norm regularization \citep{shafieezadeh2019regularization}. The norm $\|\cdot\|_*$ is the dual norm of $\|\cdot\|$ in $\mathbb{R}^d$. Note that the value of the regularization term has no dependence on the training data. Therefore in this case the coreset of WDRO logistic regression is equivalent to the coreset of the standard logistic regression.
\end{itemize}

%\textbf{Ordinary linear regression}
%The ordinary linear regression invokes the squared error $L(z)=z^2$. If $p=2$ and $\|\cdot\|$ is the Euclidean norm on $\mathbb{R}^d$, then the \textsf{WDRO} of the ordinary linear regression is equivalent to
%\begin{equation}
%\label{eq:olr}
%\begin{array}{cll}
%\inf\limits_{\theta,\lambda,s_i} & \lambda\sigma^2+\frac{1}{n}\sum_{i=1}^ns_i & \\
%\text {s.t. } & \frac{(\theta^\top x_i-y_i)^2\|\theta\|_2^2}{\lambda-\|\theta\|_2^2}\leq s_i & i\in[n]\\
%& \frac{4(\theta^\top x_i-y_i)^2}{\lambda\gamma^2-4}\leq s_i & i\in[n]\\
%& \|\theta\|_2^2 \leq \lambda,\frac{4}{\gamma^2} \leq \lambda. 
%\end{array}
%\end{equation}

%\begin{itemize}
%    \item $\kappa(\theta)=\mathtt{C}(\theta)=\|(\theta,-1)\|_*$;
%    \item $h_i(\theta,\lambda)=\max\{\frac{(\theta^\top x_i-y_i)^2\|\theta\|_2^2}{\lambda-\|\theta\|_2^2},\frac{4(\theta^\top x_i-y_i)^2}{\lambda\gamma^2-4}\}$;
%    \item $r_i(\theta,\lambda)\leq\frac{|\theta^\top x_i-y_i|\|\theta\|_2}{\lambda-\|\theta\|_2^2}+\frac{2|\theta^\top x_i-y_i|\gamma}{\lambda\gamma^2-4}$.
%\end{itemize}

\section{Experiments}
Our experiments were conducted on a server equipped with 2.4GHZ Intel CPUs and 256GB main memory. The algorithms are implemented in Python. We use the MOSEK~\citep{mosek} to solve the tractable reformulations of \textsf{WDRO}s. %Our code is available at \url{https://github.com/lwjie595/WDRO_Coreset}.

\textbf{Compared methods}
We compare our dual coreset method \textsc{DualCore} with the uniform sampling approach (\textsc{UniSamp}) and the approach that directly runs on whole dataset (\textsc{Whole}).

\textbf{Datasets}
We test the algorithms for the SVM and  logistic regression problems on two real datasets:
\textsc{Mnist}\citep{lecun2010mnist} and \textsc{Letter}\citep{chang2011libsvm}.
To simulate the scenarios where the datasets are contaminated, we perform   poisoning attacks to the training set of \textsc{Letter}. Specifically,  we use the \textsc{Min-Max} attack from \citep{DBLP:journals/corr/abs-1811-00741} and \textsc{Alfa} attack from \citep{DBLP:journals/ijon/XiaoBNXER15}. We add the standard Gaussian noise $\mathcal{N}(0,1)$ to the training set of \textsc{Mnist} and randomly flip $10\%$ of the labels.
The dual coreset algorithm for the robust regression problem is evaluated on the real dataset \textsc{Appliances Energy}\citep{candanedo2017data}.

\textbf{Remark on hyperparameter tuning in practice.} Our Algorithm \ref{alg:dualcoreset} relies on some hyperparameters like $\epsilon$ and the Wasserstein ball radius $\sigma$ (recall that the ``anchor'' $\lambda_{\texttt{anc}}$ depends on $\sigma$). By Corollary \ref{cor:monotonic} we can see that an $\epsilon$-coreset for a given $\sigma$ will still be an $\epsilon$-coreset if    the parameter $\sigma$ increases. In particular, the larger the parameter $\sigma$, the smaller the coreset size. We can use this ``monotonic'' property to avoid frequently updating our coreset by using the simple doubling technique. For example, suppose we already have the coreset for $\sigma=\sigma_0$, and we want to tune the parameter $\sigma$ gradually within a range. When $\sigma$ increases and exceeds $2\sigma_0$, we can construct the coreset of  $2\sigma_0$; if $\sigma$ exceeds $4\sigma_0$, we can construct the coreset of $4\sigma_0$, and so on and so forth. Similarly, if $\sigma$ decreases, we can try $\sigma_0/2, \sigma_0/4, \cdots$. The similar ``monotonic'' property and doubling technique also hold for $\epsilon$.

\textbf{Results}
Let $s$ and $n$ be the coreset size and the training set size, respectively. We set $c\coloneqq\frac{s}{n}$ to indicate the compression rate and fix the parameter $\gamma=7$ for all the instances (recall that $\gamma$ is used for defining the feature-label metric $\mathtt{d}(\xi_i,\xi_j)=\|x_i-x_j\|+\frac{\gamma}{2}|y_i-y_j|$). We vary the radius $\sigma$ of the Wasserstein ball for  different tasks. The experiment of  each instance were repeated by $50$ independent trials. For the \textsf{WDRO} logistic regression and SVM problems, we report the averaged test accuracy and the standard deviation in table \ref{tab:lg_mnist_acc}, \ref{tab:lg_letermm_acc}, \ref{tab:lg_leteralfa_acc}, \ref{tab:svm_leteralfa_acc} and \ref{tab:svm_letermm_acc}, where the higher accuracy of \textsc{UniSamp}  and \textsc{DualCore}  is written in bold for each instance. 
The results suggest that our dual coreset method outperforms the uniform sampling method with a higher accuracy in most cases.
For the \textsf{WDRO} robust regression task, we report the averaged test Huber loss and the standard deviation in table \ref{tab:rr_app_acc}, where the lower loss of \textsc{UniSamp}  and \textsc{DualCore}  is written in bold for each instance. The results suggest that our dual coreset method outperforms the uniform sampling method with a lower Huber loss in most cases.
We also record the normalized CPU time (over the CPU time of \textsc{Whole})  in table \ref{tab:lg_mnist_time}, \ref{tab:lg_letermm_time},  \ref{tab:lg_leteralfa_time}, \ref{tab:svm_leteralfa_time}, \ref{tab:svm_letermm_time} and \ref{tab:rr_app_time}. 
%The results on worst-case risk of Huber regression is in table \ref{tab-huber}.

\begin{table}[htbp]
	\centering
	\caption{Test set accuracy of the \textsf{WDRO} logistic regression on \textsc{Mnist} with $c=0.5\%,\sigma=0.3$}
	\begin{tabular}{c|c|c|c}
		\hline
		& \textsc{Whole}   & \textsc{UniSamp} & \textsc{DualCore} \\
		\hline      
 0 vs 1 & 99.76\% & 94.14$\pm$6.85\% & \textbf{94.37$\pm$8.05\%} \\
    0 vs 2 & 98.13\% & \textbf{87.1$\pm$10.1\%} & 84.98$\pm$11.28\% \\
    0 vs 3 & 98.78\% & 84.03$\pm$12.88\% & \textbf{84.38$\pm$12.57\%} \\
    0 vs 4 & 99.1\% & 87.82$\pm$12.79\% & \textbf{87.91$\pm$12.11\%} \\
    0 vs 5 & 97.52\% & \textbf{76.57$\pm$13.48\%} & 75.45$\pm$13.76\% \\
    0 vs 6 & 97.94\% & \textbf{83.17$\pm$12.55\%} & 81.05$\pm$12.84\% \\
    0 vs 7 & 99.27\% & 86.43$\pm$13.33\% & \textbf{90.78$\pm$8.8\%} \\
    0 vs 8 & 98.28\% & 85.57$\pm$12.99\% & \textbf{86.07$\pm$11.94\%} \\
    0 vs 9 & 98.53\% & \textbf{88.35$\pm$10.58\%} & 87.14$\pm$9.9\% \\
    1 vs 2 & 96.87\% & 79.85$\pm$13.47\% & \textbf{86.26$\pm$10.54\%} \\
    1 vs 3 & 97.58\% & \textbf{83.75$\pm$14.73\%} & 83.17$\pm$14.87\% \\
    1 vs 4 & 98.97\% & 87.73$\pm$12.73\% & \textbf{87.85$\pm$12\%} \\
    1 vs 5 & 98.03\% & 76.22$\pm$13.88\% & \textbf{81.33$\pm$13.31\%} \\
    1 vs 6 & 99\%  & 85.71$\pm$12.17\% & \textbf{89.32$\pm$8.2\%} \\
    1 vs 7 & 97.73\% & 84.2$\pm$14.22\% & \textbf{87.05$\pm$12.5\%} \\
    1 vs 8 & 95.86\% & 78.95$\pm$14.21\% & \textbf{79.13$\pm$12.93\%} \\
    1 vs 9 & 98.65\% & 87.38$\pm$13.33\% & \textbf{87.44$\pm$11.97\%} \\
    2 vs 3 & 95.91\% & 75.71$\pm$12.87\% & \textbf{76.85$\pm$12.23\%} \\
    2 vs 4 & 97.7\% & 78.12$\pm$14.26\% & \textbf{79.52$\pm$14.36\%} \\
    2 vs 5 & 96.8\% & \textbf{75.53$\pm$14.6\%} & 73.93$\pm$14.32\% \\
    2 vs 6 & 96.19\% & 72.44$\pm$12.3\% & \textbf{75.23$\pm$11.11\%} \\
    2 vs 7 & 96.39\% & 78.54$\pm$12.56\% & \textbf{86.15$\pm$10.48\%} \\
    2 vs 8 & 95.96\% & 69.08$\pm$12.62\% & \textbf{69.18$\pm$13.9\%} \\
    2 vs 9 & 97.25\% & \textbf{82.43$\pm$12.68\%} & 81.92$\pm$12.8\% \\
    3 vs 4 & 98.54\% & 77.33$\pm$16.52\% & \textbf{85.86$\pm$13.62\%} \\
    3 vs 5 & 93.4\% & 62.31$\pm$10.81\% & \textbf{66.12$\pm$10.17\%} \\
    3 vs 6 & 98.27\% & 83.06$\pm$14.52\% & \textbf{88.39$\pm$9\%} \\
    3 vs 7 & 97.41\% & 81.94$\pm$10.87\% & \textbf{81.99$\pm$14.57\%} \\
    3 vs 8 & 93.84\% & 65.56$\pm$11.81\% & \textbf{69.62$\pm$11.83\%} \\
    3 vs 9 & 96.95\% & 78.27$\pm$15.1\% & \textbf{79.13$\pm$14.22\%} \\
    4 vs 5 & 97.67\% & 70.83$\pm$13.86\% & \textbf{75.95$\pm$13.76\%} \\
    4 vs 6 & 98.2\% & 72.75$\pm$14.89\% & \textbf{72.98$\pm$14.89\%} \\
    4 vs 7 & 97.44\% & 72.55$\pm$14.45\% & \textbf{75.63$\pm$14.79\%} \\
    4 vs 8 & 98.17\% & 74.95$\pm$14.93\% & \textbf{77.88$\pm$14.33\%} \\
    4 vs 9 & 93.88\% & 59.1$\pm$9.01\% & \textbf{62.03$\pm$8.48\%} \\
    5 vs 6 & 96.88\% & 73.08$\pm$14.19\% & \textbf{77.26$\pm$13.35\%} \\
    5 vs 7 & 98.65\% & 74.25$\pm$15.56\% & \textbf{78.55$\pm$12.74\%} \\
    5 vs 8 & 93.8\% & \textbf{66.4$\pm$11.22\%} & 65.91$\pm$11.5\% \\
    5 vs 9 & 97.39\% & 70.1$\pm$14.81\% & \textbf{72.16$\pm$13.33\%} \\
    6 vs 7 & 99.49\% & 84.12$\pm$12.65\% & \textbf{87.87$\pm$11.64\%} \\
    6 vs 8 & 97.91\% & \textbf{78.85$\pm$13.94\%} & 76.58$\pm$14.58\% \\
    6 vs 9 & 99.48\% & \textbf{79.76$\pm$16.29\%} & 78.97$\pm$15.41\% \\
    7 vs 8 & 97.75\% & 79.98$\pm$14.22\% & \textbf{80.97$\pm$13.48\%} \\
    7 vs 9 & 93.14\% & 66.09$\pm$11.99\% & \textbf{67.68$\pm$12.43\%} \\
    8 vs 9 & 96.13\% & 72.12$\pm$13.57\% & \textbf{75.52$\pm$13.5\%} \\
		\hline
	\end{tabular}%
	\label{tab:lg_mnist_acc}%
\end{table}%

\begin{table}[htbp]
  \centering
  \caption{Normalized CPU time of the \textsf{WDRO} logistic regression on \textsc{Mnist} with $c=0.5\%,\sigma=0.3$}
    \begin{tabular}{c|c|c}
    		\hline
            & \textsc{UniSamp} & \textsc{DualCore} \\
          		\hline
  
    0 vs 1 & 0.01  & 0.021 \\
    0 vs 2 & 0.011 & 0.024 \\
    0 vs 3 & 0.013 & 0.029 \\
    0 vs 4 & 0.011 & 0.024 \\
    0 vs 5 & 0.008 & 0.017 \\
    0 vs 6 & 0.007 & 0.017 \\
    0 vs 7 & 0.008 & 0.018 \\
    0 vs 8 & 0.009 & 0.019 \\
    0 vs 9 & 0.011 & 0.024 \\
    1 vs 2 & 0.009 & 0.021 \\
    1 vs 3 & 0.008 & 0.018 \\
    1 vs 4 & 0.008 & 0.017 \\
    1 vs 5 & 0.007 & 0.017 \\
    1 vs 6 & 0.009 & 0.019 \\
    1 vs 7 & 0.011 & 0.023 \\
    1 vs 8 & 0.007 & 0.016 \\
    1 vs 9 & 0.008 & 0.019 \\
    2 vs 3 & 0.01  & 0.021 \\
    2 vs 4 & 0.007 & 0.017 \\
    2 vs 5 & 0.008 & 0.021 \\
    2 vs 6 & 0.01  & 0.023 \\
    2 vs 7 & 0.011 & 0.024 \\
    2 vs 8 & 0.011 & 0.024 \\
    2 vs 9 & 0.008 & 0.019 \\
    3 vs 4 & 0.008 & 0.018 \\
    3 vs 5 & 0.005 & 0.013 \\
    3 vs 6 & 0.007 & 0.016 \\
    3 vs 7 & 0.009 & 0.02 \\
    3 vs 8 & 0.01  & 0.023 \\
    3 vs 9 & 0.01  & 0.021 \\
    4 vs 5 & 0.01  & 0.021 \\
    4 vs 6 & 0.009 & 0.02 \\
    4 vs 7 & 0.007 & 0.016 \\
    4 vs 8 & 0.008 & 0.018 \\
    4 vs 9 & 0.007 & 0.017 \\
    5 vs 6 & 0.01  & 0.021 \\
    5 vs 7 & 0.011 & 0.025 \\
    5 vs 8 & 0.01  & 0.023 \\
    5 vs 9 & 0.012 & 0.025 \\
    6 vs 7 & 0.008 & 0.018 \\
    6 vs 8 & 0.009 & 0.021 \\
    6 vs 9 & 0.007 & 0.015 \\
    7 vs 8 & 0.01  & 0.021 \\
    7 vs 9 & 0.009 & 0.019 \\
    8 vs 9 & 0.008 & 0.018 \\
    		\hline
    \end{tabular}%
  \label{tab:lg_mnist_time}%
\end{table}%

\begin{table}[htbp]
  \centering
  \caption{Test set accuracy of the \textsf{WDRO} logistic regression on \textsc{Letter} under \textsc{Min-Max} attack with $\sigma=0.3$}
    \begin{tabular}{c|c|c}
    \hline
    $c$ & \textsc{UniSamp} & \textsc{DualCore} \\
      \hline
    1\%  & 79.15$\pm$14.32\% & \textbf{83.86$\pm$9.67\%} \\
    2\%  & 87.66$\pm$8.74\% & \textbf{87.81$\pm$7.03\%} \\
    3\%  & 89.32$\pm$4.44\% & \textbf{89.54$\pm$7.89\%} \\
    4\%  & 89.71$\pm$5.28\% & \textbf{90.1$\pm$5.06\%} \\
    5\%  & 90.52$\pm$4.29\% & \textbf{91.49$\pm$4.1\%} \\
    6\%  & 91.55$\pm$3.63\% & \textbf{92.36$\pm$2.56\%} \\
    7\%  & 91.19$\pm$3.68\% & \textbf{91.67$\pm$2.92\%} \\
    8\%  & \textbf{92.51$\pm$2.82\%} & 91.59$\pm$3.01\% \\
    9\%  & \textbf{92.33$\pm$2.75\%} & 91.57$\pm$2.56\% \\
    10\%   & 91.86$\pm$2.79\% & \textbf{92.57$\pm$2.08}\% \\
      \hline
    \end{tabular}%
  \label{tab:lg_letermm_acc}%
\end{table}%

\begin{table}[htbp]
  \centering
  \caption{Normalized CPU time of the \textsf{WDRO} logistic regression on the \textsc{Letter}  under \textsc{Min-Max} attack with $\sigma=0.3$}
    \begin{tabular}{c|c|c}
     \hline
    $c$ & \textsc{UniSamp} & \textsc{DualCore} \\
     \hline
    1\%  & 0.04  & 0.103 \\
    2\%  & 0.053 & 0.13 \\
    3\%  & 0.062 & 0.151 \\
    4\%  & 0.084 & 0.185 \\
    5\%  & 0.105 & 0.237 \\
    6\%  & 0.118 & 0.257 \\
    7\%  & 0.143 & 0.329 \\
    8\%  & 0.164 & 0.346 \\
    9\%  & 0.132 & 0.278 \\
    10\%   & 0.121 & 0.275 \\
     \hline
    \end{tabular}%
  \label{tab:lg_letermm_time}%
\end{table}%

\begin{table}[htbp]
  \centering
  \caption{Test set accuracy of the \textsf{WDRO} logistic regression on \textsc{Letter} under \textsc{Alfa} attack with $\sigma=0.3$}
    \begin{tabular}{c|c|c}
       \hline
    $c$ & \textsc{UniSamp} & \textsc{DualCore} \\
       \hline
    1\%  & 78.28$\pm$12.22\% & \textbf{79.86$\pm$13.56\%} \\
    2\%  & 79.69$\pm$11.74\% & \textbf{83.17$\pm$10.4\%} \\
    3\%  & 81.98$\pm$13.56\% & \textbf{84.89$\pm$10.4\%} \\
    4\%  & 87.06$\pm$8.89\% & \textbf{87.63$\pm$6.38\%} \\
    5\%  & 86.14$\pm$9.29\% & \textbf{87.16$\pm$8.53\%} \\
    6\%  & 86.9$\pm$7.44\% & \textbf{88.59$\pm$6.45\%} \\
    7\%  & \textbf{87.9$\pm$7.08\%} & 86.86$\pm$6.85\% \\
    8\%  & 88.23$\pm$5.22\% & \textbf{88.39$\pm$4.52\%} \\
    9\%  & 88.18$\pm$5.67\% & \textbf{88.63$\pm$4.43\%} \\
    10\%   & \textbf{89.33$\pm$6.46\%} & 87.44$\pm$5.05\% \\
       \hline
    \end{tabular}%
  \label{tab:lg_leteralfa_acc}%
\end{table}%

\begin{table}[htbp]
  \centering
  \caption{Normalized CPU time of the \textsf{WDRO} logistic regression on \textsc{Letter}  under \textsc{Alfa} attack with $\sigma=0.3$}
    \begin{tabular}{c|c|c}
         \hline
    $c$ & \textsc{UniSamp} & \textsc{DualCore} \\
         \hline
    1\%  & 0.03  & 0.067 \\
    2\%  & 0.033 & 0.076 \\
    3\%  & 0.041 & 0.095 \\
    4\%  & 0.045 & 0.109 \\
    5\%  & 0.058 & 0.122 \\
    6\%  & 0.06  & 0.137 \\
    7\%  & 0.072 & 0.152 \\
    8\%  & 0.092 & 0.201 \\
    9\%  & 0.125 & 0.25 \\
    10\%   & 0.099 & 0.217 \\
         \hline
    \end{tabular}%
  \label{tab:lg_leteralfa_time}%
\end{table}%

\begin{table}[htbp]
  \centering
  \caption{Test set accuracy of the \textsf{WDRO} SVM on \textsc{Letter} under \textsc{Alfa} attack with $\sigma=0.1$}
    \begin{tabular}{c|c|c}
    \hline
    $c$ & \textsc{UniSamp} & \textsc{DualCore} \\
    \hline
    1\%   & \textbf{80.29$\pm$13.98\%} & 79.23$\pm$13\% \\
    2\%   & 83.86$\pm$13.7\% & \textbf{87.8$\pm$11.74\%} \\
    3\%   & 89.95$\pm$11.23\% & \textbf{92.78$\pm$7.61\%} \\
    4\%   & 91.47$\pm$9.37\% & \textbf{92.46$\pm$6.54\%} \\
    5\%   & 90.89$\pm$9.12\% & \textbf{92.36$\pm$9.05\%} \\
    6\%   & \textbf{95.5$\pm$4.35\%} & 94.94$\pm$5.49\% \\
    7\%   & 94.01$\pm$6.85\% & \textbf{95.99$\pm$2.65\%} \\
    8\%   & 95.61$\pm$5.7\% & \textbf{96.1$\pm$2.43\%} \\
    9\%   & 94.91$\pm$6.09\% & \textbf{96.43$\pm$2.23\%} \\
    10\%  & 95.27$\pm$5.72\% & \textbf{95.97$\pm$3.76\%} \\
    \hline
    \end{tabular}%
  \label{tab:svm_leteralfa_acc}%
\end{table}%

\begin{table}[htbp]
  \centering
  \caption{Normalized CPU time of the \textsf{WDRO} SVM on \textsc{Letter}  under \textsc{Alfa} attack with $\sigma=0.1$}
    \begin{tabular}{c|c|c}
    \hline
    $c$ & \textsc{UniSamp} & \textsc{DualCore} \\
      \hline
    1\%   & 0.137 & 0.554 \\
    2\%   & 0.137 & 0.561 \\
    3\%   & 0.121 & 0.546 \\
    4\%   & 0.125 & 0.501 \\
    5\%   & 0.133 & 0.503 \\
    6\%   & 0.168 & 0.515 \\
    7\%   & 0.228 & 0.695 \\
    8\%   & 0.349 & 0.955 \\
    9\%   & 0.327 & 0.984 \\
    10\%  & 0.163 & 0.556 \\
      \hline
    \end{tabular}%
  \label{tab:svm_leteralfa_time}%
\end{table}%

\begin{table}[htbp]
  \centering
  \caption{Test set accuracy of the \textsf{WDRO} SVM on \textsc{Letter} under \textsc{Min-Max} attack with $\sigma=0.2$}
    \begin{tabular}{c|c|c}
     \hline
    $c$ & \textsc{UniSamp} & \textsc{DualCore} \\
     \hline
    1\%   & 82.02$\pm$15.31\% & \textbf{85.17$\pm$13.88\%} \\
    2\%   & 90.44$\pm$9.08\% & \textbf{93.13$\pm$2.26\%} \\
    3\%   & 90.29$\pm$10.71\% & \textbf{92.17$\pm$6.62\%} \\
    4\%   & 91.29$\pm$9.01\% & \textbf{93.7$\pm$2.26\%} \\
    5\%   & 93.55$\pm$2.43\% & \textbf{93.91$\pm$1.72\%} \\
    6\%   & \textbf{94.17$\pm$2.18\%} & 93.13$\pm$6.47\% \\
    7\%   & 92.68$\pm$7.33\% & \textbf{94.39$\pm$1.47\%} \\
    8\%   & 94.15$\pm$2.05\% & \textbf{94.2$\pm$1.36\%} \\
    9\%   & \textbf{94.26$\pm$1.61\%} & 94.04$\pm$1.28\% \\
    10\%  & 93.99$\pm$1.53\% & \textbf{94.2$\pm$1.67\%} \\
     \hline
    \end{tabular}%
  \label{tab:svm_letermm_acc}%
\end{table}%

\begin{table}[htbp]
  \centering
  \caption{Normalized CPU time of the \textsf{WDRO} SVM on \textsc{Letter}  under \textsc{Min-Max} attack with $\sigma=0.2$}
    \begin{tabular}{c|c|c}
    \hline
    $c$ & \textsc{UniSamp} & \textsc{DualCore} \\
    \hline
    1\%   & 0.118 & 0.517 \\
    2\%   & 0.122 & 0.516 \\
    3\%   & 0.11  & 0.445 \\
    4\%   & 0.118 & 0.458 \\
    5\%   & 0.159 & 0.537 \\
    6\%   & 0.245 & 0.804 \\
    7\%   & 0.33  & 0.965 \\
    8\%   & 0.297 & 0.871 \\
    9\%   & 0.159 & 0.519 \\
    10\%  & 0.218 & 0.647 \\
    \hline
    \end{tabular}%
  \label{tab:svm_letermm_time}%
\end{table}%

\begin{table}[htbp]
  \centering
  \caption{Test set Huber loss of the \textsf{WDRO} robust regression on \textsc{Appliances Energy} with $\sigma=100$}
    \begin{tabular}{c|c|c}
        \hline
    $c$ & \textsc{UniSamp} & \textsc{DualCore} \\
        \hline
    1\%   & 33.0933$\pm$1.8918 & \textbf{32.3245$\pm$1.937} \\
    2\%   & 31.4399$\pm$1.5614 & \textbf{30.7886$\pm$1.2459} \\
    3\%   & 31.3852$\pm$0.6885 & \textbf{30.5185$\pm$0.4625} \\
    4\%   & 31.5143$\pm$0.4824 & \textbf{31.0308$\pm$0.3113} \\
    5\%   & 31.036$\pm$0.507 & \textbf{30.401$\pm$0.2476} \\
    6\%   & 31.5388$\pm$0.3296 & \textbf{31.0017$\pm$0.1913} \\
    7\%   & 32.2394$\pm$0.311 & \textbf{31.7412$\pm$0.1504} \\
    8\%   & 30.225$\pm$0.2345 & \textbf{29.8135$\pm$0.1503} \\
    9\%   & 30.0463$\pm$0.2292 & \textbf{29.6167$\pm$0.1098} \\
    10\%  & 31.1906$\pm$0.2257 & \textbf{30.8201$\pm$0.107} \\
        \hline
    \end{tabular}%
  \label{tab:rr_app_acc}%
\end{table}%

\begin{table}[htbp]
  \centering
  \caption{Normalized CPU time of the \textsf{WDRO} robust regression on \textsc{Appliances Energy} with $\sigma=100$}
    \begin{tabular}{c|c|c}
    \hline
    $c$ & \textsc{UniSamp} & \textsc{DualCore} \\
     \hline
    1\%   & 0.039 & 0.145 \\
    2\%   & 0.076 & 0.224 \\
    3\%   & 0.068 & 0.192 \\
    4\%   & 0.072 & 0.207 \\
    5\%   & 0.088 & 0.236 \\
    6\%   & 0.097 & 0.244 \\
    7\%   & 0.103 & 0.259 \\
    8\%   & 0.157 & 0.348 \\
    9\%   & 0.147 & 0.299 \\
    10\%  & 0.184 & 0.374 \\
     \hline
    \end{tabular}%
  \label{tab:rr_app_time}%
\end{table}%

\end{document}